\documentclass[twoside]{article}

\usepackage[accepted]{aistats2026}

\usepackage[round]{natbib}

\usepackage{booktabs}
\usepackage{mathtools}
\usepackage{amssymb}
\usepackage{url}
\usepackage[ruled]{algorithm2e}
\SetKwComment{Comment}{}{}

\usepackage{amsthm}

\DeclareMathOperator{\bJ}{\boldsymbol{J}}
\DeclareMathOperator{\bx}{\boldsymbol{x}}

\DeclareMathOperator{\bz}{\boldsymbol{z}}
\DeclareMathOperator{\btheta}{\boldsymbol{\theta}}
\DeclareMathOperator{\R}{\mathbb{R}}
\DeclareMathOperator{\E}{\mathbb{E}}
\DeclareMathOperator{\bbf}{\boldsymbol{f}}
\DeclareMathOperator{\bbv}{\boldsymbol{v}}
\DeclareMathOperator{\bc}{\boldsymbol{c}}

\DeclareMathOperator{\br}{\boldsymbol{r}}
\DeclareMathOperator{\bg}{\boldsymbol{g}}
\DeclareMathOperator{\bG}{\boldsymbol{G}}
\DeclareMathOperator{\bX}{\boldsymbol{X}}
\DeclareMathOperator{\bZ}{\boldsymbol{Z}}
\DeclareMathOperator{\ba}{\boldsymbol{a}}
\DeclareMathOperator{\bh}{\boldsymbol{h}}
\DeclareMathOperator{\M}{\mathcal{M}}
\DeclareMathOperator{\bbc}{\boldsymbol{c}}
\DeclareMathOperator{\be}{\boldsymbol{e}}
\DeclareMathOperator{\bU}{\boldsymbol{U}}
\DeclareMathOperator{\bu}{\boldsymbol{u}}
\DeclareMathOperator{\W}{\mathcal{W}}
\DeclareMathOperator{\bS}{\boldsymbol{S}}
\DeclareMathOperator{\bmu}{\boldsymbol{\mu}}
\DeclareMathOperator{\bA}{\boldsymbol{A}}
\DeclarePairedDelimiter{\norm}{\lVert}{\rVert}

\newcommand{\singlem}{\textsc{Single}-$\M$}
\newcommand{\single}{\textsc{Single}}
\newcommand{\multdirect}{\textsc{Mult-Direct}}
\newcommand{\mult}{\textsc{Mult}}

\newtheorem{theorem}{Theorem}

\begin{document}

\runningauthor{Hanlin Yu, Søren Hauberg, Marcelo Hartmann, Arto Klami, Georgios Arvanitidis}

\twocolumn[

\aistatstitle{Learning Geometry and Topology via Multi-Chart Flows}

\aistatsauthor{ Hanlin Yu \And Søren Hauberg \And Marcelo Hartmann }

\aistatsaddress{ University of Helsinki \And Technical University of Denmark \And University of Helsinki }

\aistatsauthor{ Arto Klami \And Georgios Arvanitidis}

\aistatsaddress{ University of Helsinki \And Technical University of Denmark } ]

\begin{abstract}
  Real world data often lie on low-dimensional Riemannian manifolds embedded in high-dimensional spaces. This motivates learning degenerate normalizing flows that map between the ambient space and a low-dimensional latent space. However, if the manifold has a non-trivial topology, it can never be correctly learned using a single flow. Instead multiple flows must be `glued together'. In this paper, we first propose the general training scheme for learning such a collection of flows, and secondly we develop the first numerical algorithms for computing geodesics on such manifolds. Empirically, we demonstrate that this leads to highly significant improvements in topology estimation.
\end{abstract}

\section{INTRODUCTION}

Normalizing flows \citep{Papamakarios2021flows} are a family of flexible neural network models of complex probability distributions from finite observations. They are bijections that allow tractable sampling and log-density evaluations. Due to the bijectivity requirement, the dimensionalities of the latent and ambient spaces need to be identical. As real-world datasets often lie on a low-dimensional Riemannian manifold embedded in a high-dimensional ambient space, we usually do not expect a bijection. As such, the modeling strategy of the flows can be ill-posed. When the manifold structure is known beforehand, one can construct specific flows that adapt to it \citep{rezende20flows_tori_spheres,ng2023mixture_flows}.

\begin{figure}[h]
\centering
\includegraphics[width=0.3\textwidth]{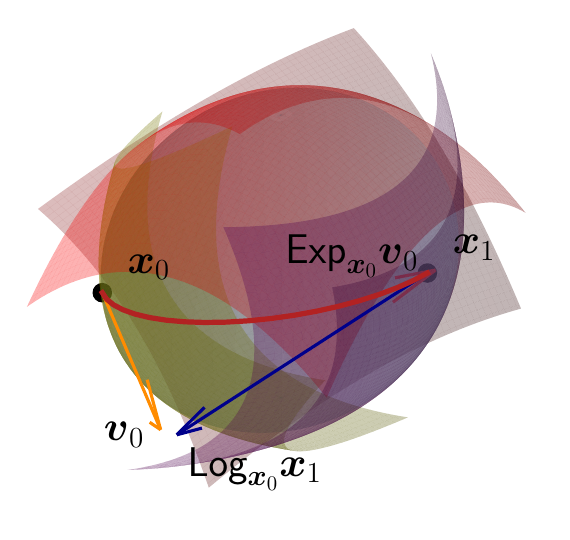}
\vspace{-7mm}
\caption{Using multiple charts to cover a manifold gives the ability to learn the manifold's true geometry and topology. The charts may not perfectly align, requesting careful treatments.}
\label{fig:teaser}
\end{figure}

A probability density on an unknown Riemannian manifold can be modeled using a degenerate normalizing flow \citep{Brehmer2020mflows,Caterini2021rectangular}, by learning a mapping between a low-dimensional latent space and a high-dimensional ambient space. It was demonstrated that, for certain manifolds, the flows are able to model the distribution accurately. However, these methods often consider a single normalizing flow, which is justified only when the manifold is diffeomorphic to a Euclidean space. Some common manifolds, for instance the sphere and the torus, do not satisfy this assumption. It is by definition impossible for a single-chart flow to model these manifolds correctly due to the topological mismatch.

Understanding the underlying topology and geometry of the data manifold using only finite observations is an open problem in machine learning \citep{hauberg2019only,moor2020topo_ae,acosta2023quantifying,ross2024implicit_manifold,diepeveen2024pulling}. Some attempts focus on capturing the Riemannian structure using the pullback metric induced by a single normalizing flow \citep{hoffman2019neutra,kruiff2024pullback}, while other works \citep{schonsheck2020chart,Kalatzis2022multi,sidheekh2022vq-flows,alberti2024manifold} use multiple autoencoders or normalizing flows to cover the manifold. However, these approaches often lack a principled way of choosing the correct encoder \citep{loaiza-ganem2024deep}. Most importantly, these works did not study the geometric and topological implications of the manifold represented by the generative model.

\textbf{In this paper}, we aim at learning data manifolds with the correct geometry and topology. 
Classic differential geometry achieves this by
gluing together local ``patches''
of the manifold, each of which is a diffeomorphism. Similarly, we learn a mixture of flows, each corresponding to a ``patch''. In classic differential geometry, patches are assumed to overlap perfectly in bordering regions, but such theoretical constructions are unrealistic in practice using flows (see Figure~\ref{fig:teaser}). We propose a formalism to handle such non-overlapping patches, and further develop algorithms for computing geodesics over manifolds defined through such mixture of flows. Our contributions are:

1. We demonstrate that using a single normalizing flow to model a distribution on a Riemannian manifold can lead to significant misunderstanding in terms of the underlying topology and geometry.

2. We provide general training strategies for multi-chart flows based on the maximum likelihood principle and demonstrate that they can perform better than single-chart flows of similar complexities. While both gluing together multiple degenerate normalizing flows \citep{Kalatzis2022multi,sidheekh2022vq-flows} and training a mixture of flows using Expectation Maximization (EM) \citep{ng2023mixture_flows} have been explored before, to the best of our knowledge, our work is the first to formulate the Maximum Likelihood Estimation / EM algorithm for directly training such a mixture of degenerate flows which also provides a principled way to choose the flows through the responsibilities.

3. While previous works generally focus on modeling and density estimation and \textbf{did not} study how to compute the exponential maps and logarithmic maps, we develop geodesic algorithms on the learned manifold. We provide specialized tools to utilize the geometry learned by multi-chart flows that crucially respect the actual topology of the underlying manifold, taking into account that the individual flows may not strictly agree with each other, yielding more reliable representations. This is our main contribution.

\section{PRELIMINARIES}

\subsection{A Brief Intro to Riemannian Geometry}

Riemannian manifolds \citep{docarmo1992,Lee2018rm} are spaces that locally resemble Euclidean spaces. We consider a $d$-dimensional Riemannian manifold isometrically embedded in a $D$-dimensional Euclidean space, following the typographic mnemonic notation that the lower-case symbol denotes the smaller dimensionality. The well-known Nash embedding theorem \citep{Lee2018rm} guarantees that such an embedding exists for all Riemannian manifolds, though it is not an operational definition. 
A \emph{local coordinate chart} is given by the pair $(U,\phi)$, where $U$ is an open set on the manifold and $\phi$ is a bijection between $U$ and an open set of $d$-dimensional Euclidean space. 
For simple manifolds, one might be able to construct a single chart that covers the entire manifold.
However, for more complex ones, we need a collection of local charts that transition smoothly, known as an \emph{atlas}.

A curve with zero tangential acceleration, intuitively a shortest path on the manifold, is known as a \emph{geodesic}.
For a point $\bx$ on a $d$-dimensional Riemannian manifold, one can attach a $d$-dimensional Euclidean space known as the tangent space at $\bx$, containing all vectors that are tangent to the surface at $\bx$. Given the initial position $\bx_{0}$ and the initial velocity $\bbv_{0}$, one can use the exponential map $\text{Exp}_{\bx_{0}}(\bbv_{0})$ to follow along the geodesic for unit time and obtain the final position $\bx_{1}$. Given the initial position $\bx_{0}$ and the final position $\bx_{1}$, one can use the logarithmic map $\text{Log}_{\bx_{0}}(\bx_{1})$ to obtain (one of) the initial velocity that satisfies the boundary condition. These geometric operations are also illustrated in Figure~\ref{fig:teaser}, where we show that these quantities are well-defined even if the manifold is covered by multiple charts.

\subsection{Normalizing Flows}
\label{sec:normalizing_flows}

A normalizing flow \citep{Papamakarios2021flows} pushes forward a tractable distribution $p(\bu)$ from a \emph{latent space} $\bU$ to a complex distribution in the \emph{input space} $\bX$ through the bijection $\bbf$, such that the resulting distribution $q_{\btheta}(\bx)$ in $\bX$ can be tractably evaluated.

For density estimation tasks on Euclidean spaces where $\text{dim}(\bX)=\text{dim}(\bU)=D$, normalizing flows are typically trained by minimizing the forward KL divergence $\text{KL}\left(p(\bx) \vert q_{\btheta}(\bx)\right)$ \citep{Papamakarios2021flows}, where the density $q_{\btheta}(\bx)$ can be evaluated as $q_{\btheta}(\bx) = p(\bu) \vert \text{det}\left(\bJ_{\bbf}(\bu)\right) \vert^{-1}\Big|_{\bu = \bbf^{-1}(\bx)}$, with $\bJ_{\bbf}\in \mathbb{R}^{D\times D}$ being the Jacobian matrix of the parameterized transformation $\bbf: \bU \rightarrow \bX$; the dependency of $\bbf$ on $\btheta$ is dropped to avoid cluttering notations. One can employ specific architectures, such that the log-determinant of the Jacobian can be obtained efficiently. Two notable examples are RealNVP flows \citep{dinh2017realnvp} and Neural Spline flows \citep{durkan2019nsf}.

In many scenarios, one would need to model distributions that are supported on Riemannian manifolds. In the following discussions, we largely follow \citet{Brehmer2020mflows} and \citet{Caterini2021rectangular}. A normalizing flow on a Riemannian manifold is commonly parameterized as the composition of a series of $d$-dimensional layers $\bg$ and a series of a $D$-dimensional layers $\bh$, with the resulting
transformation given by
\begin{equation}
\bbf = \bh \circ \,\text{Pad} \circ \bg = \tilde{\bh} \circ \bg,
\end{equation}
where $\circ$ denotes composition of functions, $\text{Pad}$ denotes adding $D-d$ zeros at the end, and $\tilde{\bh}$ denotes $\bh \circ \,\text{Pad}$. The resulting change of variable formula is more involved \citep{Brehmer2020mflows}
\begin{equation}
q_\theta(\bx) = p(\bu) \left|\text{det}(\bJ_{\bbf}^{\top}(\bu)\bJ_{\bbf}(\bu))\right|^{-\frac{1}{2}}\Big|_{\bu = \bbf^{-1}(\bx)},
\end{equation}
where $\bJ_{\bbf}\in\R^{D\times d}$. Due to the degeneracy, without good reconstructions the model may learn suboptimal distributions despite having seemingly good log-likelihood estimates \citep{Brehmer2020mflows}.
For this reason, they proposed the Manifold-learning Flow, where $\bh$ focuses on accurately reconstructing the data and $\bg$ focuses on learning the density. 

While their mechanism is theoretically justified \citep{Loaiza-ganem2022manifold_overfitting}, \citet{Caterini2021rectangular} noted that the optimization problem becomes challenging, leading to potentially inferior performances in practice. Instead, it can be beneficial to explicitly optimize the KL divergence while penalizing the reconstruction loss. One can define $\tilde{\bh}^{\dagger} = \text{Proj} \circ \bh^{-1}$, where $\text{Proj}$ is the operation that takes the $d$ coordinates; this is the proper inverse function of $\tilde{\bh}$ when $\bx$ lies precisely on the hypersurface formed by $\tilde{\bh}$. The resulting loss function is given by
\begin{equation}
\E_{\bx\sim p(\bx)}\left[-\log q_{\btheta}(\bx) + \lambda \norm{\bx - \tilde{\bh} \left(\tilde{\bh}^{\dagger}\left(\bx\right)\right)}^{2}\right],
\end{equation}
where $\lambda > 0$ controls the trade-off between reconstruction and density estimation. 

Recall that, to obtain $\log q_{\btheta}(\bx)$, we need to compute $\log\text{det}(\bJ_{\bbf}^{\top}\bJ_{\bbf})$. \citet{Caterini2021rectangular} showed that this quantity can be simplified further as
\begin{equation}
\log\det\left(\bJ_{\bbf}^{\top}\bJ_{\bbf}\right) = 2\log\det\left(\bJ_{\bg}\right) + \log\det\left(\bJ_{\tilde{\bh}}^{\top} \bJ_{\tilde{\bh}}\right),
\end{equation}
where $\log\det\left(\bJ_{\bg}\right)$ is tractable as $\bg$ is a square flow, and only $\log\det\left(\bJ_{\tilde{\bh}}^{\top} \bJ_{\tilde{\bh}}\right)$ is a complex term.
For optimization we only need the gradient of the log-likelihood. With reasonably small $d$ as in many practical applications, we can compute the $D \times d$ Jacobian matrix exactly using vectorization and forward mode automatic-differentiation with deep learning frameworks, while the cost of obtaining the log-determinant for $d \times d$ is affordable. For larger scale problems, \citet{Caterini2021rectangular} proposed an estimator based on Hutchinson trace estimator and conjugate gradients, which can potentially speed up optimizations when the number of conjugate gradient steps is small enough and $d$ is large enough. However, we observed that the exact estimate is largely practical and employed it.

\subsection{Persistent Homology}

In order to measure whether we succeeded in learning and engaging with manifolds with non-trivial topology, we rely on the tool of topological data analysis. We briefly introduce the concept and tools for persistent homology based on \citet{wasserman2018tda}. Data can reflect the topological structures of the underlying space, and persistent homology provides a principled tool to capture these topological features, often in the form of a persistence diagram. One usually cares about topological features including $H_{0}$, $H_{1}$ and $H_{2}$, where $H_{0}$ refers to the number of connected components, $H_{1}$ and $H_{2}$ refer to the number of one and two-dimensional holes, respectively. 
A persistence diagram is plotted by monitoring the births and deaths of topological features as balls of varying radii are drawn around data points, with a longer life span implying a more significant feature.
A persistence diagram can be generated from a matrix containing pairwise distances. While pairwise Euclidean distances can be used here, the intrinsic distances may be preferred \citep{fernandez2023intrinsic}. In this work, inspired by the later approach, we use the geodesic distances.

\section{GEOMETRY OF SINGLE-CHART FLOWS}

Largely neglected by previous works concerning normalizing flows on Riemannian manifolds, the flow enables us to learn and utilize the underlying geometric structure of the manifold.

\paragraph{The Learned Manifold.} A fundamental limitation of a single-chart flow is that it assumes that the manifold is globally diffeomorphic to Euclidean. This is an assumption that is violated by many common manifolds, including sphere and torus. Nevertheless, given a trained normalizing flow, one can reason about the learned Riemannian manifold.

Normalizing flows are by definition bijections, so a flow acts as a coordinate chart in the Riemannian sense, and naturally induces a well-defined pullback metric in the latent space $\bU$. We further observe that only $\tilde{\bh}$ contributes to the manifold embedding, while $\bg$ is just a reparametrization $\bZ = \bg(\bU)$ of the latent space. Hence, we denote $\bz= \tilde{\bh}^{\dagger}(\bx)$, and the Riemannian metric in $\bZ$ takes the form
\begin{align}
\label{eq:riemannian_metric}
\bG_{\bZ}\left(\bz\right) = \left(\frac{\partial \bx}{\partial \bz}\right)^{\top} \bG_{\bX}\left(\bx\right) \left(\frac{\partial\bx}{\partial\bz}\right) = \left(\frac{\partial\bx}{\partial\bz}\right)^{\top} \left(\frac{\partial\bx}{\partial\bz}\right),
\end{align}
where $\bG_{\bX}(\bx)=\mathbb{I}_D$ is the ambient Euclidean metric. Furthermore, the Jacobian $\frac{\partial \bx}{\partial \bz}\in\mathbb{R}^{D\times d}$ maps a vector from the latent space to a vector in the ambient space as $\bbv_{\bX}\left(\bx\right) = \frac{\partial \bx}{\partial \bz}\bbv_{\bZ}\left(\bz\right)$. In practice, the metric can be obtained through explicitly calculating the Jacobian matrices, while the transformation between the vectors can be obtained through a Jacobian vector product operation, which is reasonably cheap for modern automatic-differentiation systems \citep{baydin2018ad}. Further discussions on the geometric interpretation can be found in 
Section~\ref{sec:app-geo_details}.

As discussed next, the flow enables us to compute geodesics, which are essential for many practical operations on a Riemannian manifold, such as exponential maps and logarithmic maps. We remark that one can in principle obtain other geometric quantities, e.g. the Riemannian curvature, through performing the relevant operations on the learned metric.

\paragraph{Exponential Maps.}
\label{sec:single-exp}
Using the Riemannian metric \eqref{eq:riemannian_metric}, one can solve the exponential map by solving the geodesic equation as explicitly formulated below.

\begin{theorem}
The geodesic equation in the latent space induced by the pullback metric \eqref{eq:riemannian_metric} is given by
\begin{equation}
\ddot{z}(t)^{k} = - g^{kl} \sum_{i,j,m} \frac{\partial^{2}x_{m}}{\partial z_{i}\partial z_{j}} \frac{\partial x_{m}}{\partial z_{l}} \dot{z}^{i} \dot{z}^{j}.
\end{equation}
\label{thm:single-geodesic}
\end{theorem}
The proof can be found in 
Section~\ref{sec:app-exp-maps-proof}.
We remark that this is similar to the geodesic induced by Fisher information matrix with Gaussian likelihood \citep{song2018accelerating}.
The algorithm for solving exponential maps involves explicitly forming the Jacobian matrix, calculating the Riemannian metric and its inverse, thus has complexity $O(d^{3})$. However, $d$ can be much smaller than the dimensionality of the ambient space $D$. 
Moreover, one can avoid explicitly calculating and storing the higher order gradient tensors. Further discussions can be found in
Section~\ref{sec:app_imp_geo}.

\paragraph{Geodesics and Logarithmic Maps.}
\label{sec:single-log}
A geodesic corresponds to a curve with locally minimum energy. One can thus parameterize a curve using a cubic spline in the latent space of the flow and optimize its parameters to minimize the energy as observed in the ambient space \citep{yang:arxiv:2018, detlefsen2021stochman}. Having the optimized curve, we can estimate the distance between the two end points using discretization techniques, and the logarithmic map as the initial velocity.

\paragraph{Pathology of Single-Chart Flows.}
\label{sec:pathology-single}

A single-chart flow is, by definition, a bijection between a $d$-dimensional Euclidean space and the learned manifold. As such, the learned manifold is constrained to be diffeomorphic to a Euclidean space. This is a fundamental limitation: when the underlying data manifold is not diffeomorphic to Euclidean, the flow cannot reliably represent its global structure. Perhaps the simplest example is the circle: although it can be parameterized by a single angular coordinate, the fact that $a$ and $a + 2\pi$ refer to the same point for any $a \in \R$ means it is not Euclidean. In practice, covering the circle with a single flow inevitably introduces a small ``gap'' in the learned manifold. As a result, the implied geodesics are necessarily suboptimal as a geodesic needs to traverse the entire circle to connect the two sides of the gap. In addition, the implied topology is wrong.

\section{MODELING USING MULTI-CHART FLOWS}

We resolve the fundamental issue mentioned above, namely, the limitations caused by using a single flow, by considering a mixture of flows that are carefully designed to properly cover the data manifold while respecting its geometric and topological structure. In particular, we aim for each flow to focus on a specific region of the manifold, and we rely on a probabilistic formulation to ensure that overlapping regions exist which is crucial when covering a manifold. We then leverage this construction to extend standard geometric computations, such as geodesics, to this data-driven manifold. We note that the main goal of the proposed generative model is to provide the suitable foundation upon which geometric quantities can be computed, while preserving the intrinsic properties of the manifold.

\paragraph{Model.}
The previous discussions suggest that we need to use multiple flows instead of a single one. We thus consider the following mixture model
\begin{equation}
\sum_{i=1}^{N}\log q_{\btheta}(\bx_{i}) = \sum_{i=1}^{N}\log\left(\sum_{c=1}^{C} q_{\btheta}(\bx_{i}|c) q(c)\right),
\end{equation}
where $N$ denotes the number of data points, $C\in\mathbb{Z}^+$ denotes the number of charts, and each $q_{\btheta}(\bx|c)$ is a normalizing flow with latent dimensionality $d$ and ambient dimensionality $D$ for chart $c$.

\paragraph{Training.}
With the mixture model formulation, our goal is to minimize the distance between the model distribution and the ground truth distribution. 
We can write the resulting KL divergence as

\begin{equation}
\text{KL}\left(p(\bx) \vert q_{\btheta}(\bx)\right) = \text{KL}\left(p(\bx) \left\vert \sum_{c=1}^{C} q_{\btheta}(\bx|c) q(c)\right. \right).
\end{equation}

However, as noted by previous works \citep{Brehmer2020mflows,Caterini2021rectangular}, purely performing Maximum Likelihood Estimation (MLE) training is not sufficient for fitting the degenerate flow to a distribution supported on the manifold. Due to degeneracy, the log-density as calculated by the normalizing flow is in fact the log-density projected onto the learned manifold. Therefore, inspired by \citet{Caterini2021rectangular}, we calculate the required quantities by adding a regularization term to the log-density of the flow. The resulting log-density for a single data point is given by
\begin{equation}
\label{eq:mixture_log_component}
\log \tilde{q}_{\btheta}(\bx_{i}|c) = \log q_{\btheta}(\bx_{i}|c) - \lambda \norm{\bx_{i} - \tilde{\bh}_{c} \left(\tilde{\bh}_{c}^{\dagger}\left(\bx_{i}\right)\right)}^{2}.
\end{equation}
Depending on the problem one may add additional regularization terms here. Upon optimality, one can expect the different flows to learn essentially the same reconstructions in overlapping regions, while modeling the target density through a mixture model.
There can be potentially an infinite number of models that fit the manifold and the distribution perfectly.

For training a mixture model, there are generally two commonly employed approaches, i.e. directly performing Maximum Likelihood Estimation (MLE) and using Expectation Maximization (EM). These approaches can be employed in our case as well. In this work, we always employ a uniform prior over $q(c)$, so as to encourage the individual flows to cover approximately equal amount of space and prevent them from not being responsible for any points, leading to a waste of model capacity.

\paragraph{Maximum Likelihood Estimation (MLE).}
Using $\log \tilde{q}_{\btheta}(\bx_{i}|c)$ as in \ref{eq:mixture_log_component},  one can define
\begin{equation}
\log \tilde{q}_{\btheta}(\bx) = \sum_{i=1}^{N}\log \sum_{c=1}^{C} \exp\left(\log \tilde{q}_{\btheta}(\bx_{i}|c) + \log q(c)\right),
\end{equation}
which can be implemented using the \textit{logsumexp} operation, enabling direct MLE training of the flows.

\paragraph{Expectation Maximization (EM).}
EM is arguably the most commonly used algorithm for training mixture models, and we adapt the EM algorithm to train multi-chart flows. During EM training, the E-step and the M-step are iteratively applied. In one E-step, the responsibilities of the individual flows can be calculated as
\begin{equation}
r_{c}(\bx_{i}) = \text{stop\_gradient}\left(\text{SoftMax}\left(\log \tilde{q}_{\btheta}(\bx_{i}|c)\right)\right),
\end{equation}
while in one M-step, one minimizes the loss function
\begin{equation}
L(\btheta) = -\sum_{i=1}^{N}\sum_{c=1}^{C}r_{c}(\bx_{i}) \log \tilde{q}_{\btheta}(\bx_{i}|c).
\end{equation}
Further discussions can be found in 
Section~\ref{sec:app-em}.
We employ a single gradient descent update for the M-step. 
One cannot expect a flow to correctly model data points that are far from its responsible area. As such, for an individual flow we only calculate the loss based on points where its responsibilities are above a threshold. The probabilistic formulation using responsibilities allows the charts to overlap in some regions.

\paragraph{General Strategy.}

In preliminary experiments, we did not observe fundamental differences between direct MLE and EM. However, the responsibilities as provided by EM offer interpretability and ways to regularize the flows. In this work, we focus on EM. For training mixture models, one common strategy is to initialize the clusters using simple clustering algorithms like K-Means \citep{bishop2006prml}. Similarly, we apply K-Means to obtain $C$ clusters, training each flow individually within the clusters before running EM.

We remark that training a mixture of normalizing flows has been explored before, where both MLE and EM have been demonstrated to work \citep{izmailov2020semi_supervised,atanov2020semi_conditional,ng2023mixture_flows}. Nevertheless, it is less investigated as for how to train mixtures of flows when both the data manifold and the distribution are unknown.

\section{GEOMETRY OF MULTI-CHART FLOWS}

\paragraph{The Learned Manifold.}

Similar to the case of classical differential geometry, with multi-chart flows, there may not exist a unique latent space. However, one can still obtain the correct geometric structure from all of them. When the flows are perfectly trained, each individual flow acts as a local chart due to its bijective nature, and we obtain a well-defined atlas that covers the entire manifold by collecting all the charts. However, in practice, we only have access to a finite and potentially noisy dataset, the modeling and optimization are most likely imperfect, and one \textit{cannot} expect the flows to be perfect.

We argue that we can still reason geometrically through \textit{inconsistent} flows by utilizing a probabilistic treatment, which has been demonstrated useful for learning meaningful geometry in VAEs \citep{arvanitidis2018latent,hauberg2019only}. Denote the value of interest as $\be$. Given $C$ trained flows, since each flow would give us an estimate, we can collect all these estimates to form a set $\{\be_{c},c=1,\dots ,C\}$, and the task is to assign the right weight to each $\be_{c}$. From a Bayesian perspective, a natural approach to obtain a single estimate $\be$ is to use the predictive posterior $\E_{p(c|\bx)}\left[\be_{c}\right]$.

Interestingly, these weights are precisely the responsibilities given by the EM algorithm.  As such, we define the points that lie on the manifold as the set $\{\sum_{c=1,\dots ,C} r_{c}(\bx') \tilde{\bh}_{c}\left(\tilde{\bh}_{c}^{\dagger}\left(\bx'\right)\right), ~\forall \bx' \in \mathbb{R}^{D}\}$. This can be seen as a projection of the points $\bx'$ in the ambient space onto the learned manifold, which is well-defined for sufficiently close points $\bx'$. We can similarly define the tangent space at point $\bx$ as the tangent spaces of different charts weighted by $r(\bx)$. When all the flows yield the correct estimate everywhere, these estimates correspond to the ground truth values. Otherwise, the probabilistic treatment results in more robust estimates. In practice, we may set a threshold on the responsibilities and form the predictive only based on those that are above the threshold. Since the flows are trained based on them, we can expect the responsibilities to yield well-defined behavior. \citet{alberti2024manifold} also observed the benefit of averaging, albeit in the context of Riemannian optimizations.

\paragraph{Exponential Maps.}
\label{sec:mult-exp}
Due to the involved definition of manifold structure as discussed above, we need specialized tools to compute the exponential map. Perhaps the most natural way is to apply an algorithm inspired by classical differential geometry, where one solves the exponential map based only on the chart with maximal responsibility, and jumps to another when the most responsible one changes. However, as the responsibilities of the charts decrease, the estimates of the charts can degrade in different ways. Practically, the numerical solvers make it even more challenging. Therefore, we argue that it is beneficial to solve the exponential maps using the responsibilities as calculated in the EM algorithm as weights.

Algorithm~\ref{alg:mult-exp-euler} provides an Euler-like method.
We compute a unique curve directly in the ambient space, and instead of relying exclusively on the exponential map from a single chart, each update step is computed as the weighted average over multiple charts using the corresponding responsibilities. In each responsible chart, we simulate the geodesic for $\Delta t$ time starting from $\bx_{t}$ and $\bbv_{t}$, which can be implemented using any black-box ODE solvers through numerically integrating the geodesic ODE. We provide further ablation studies on the different solvers in Section~\ref{sec:app-exp-ints}, verifying that averaging is crucial for a well-behaved solver.

\paragraph{Geodesics and Logarithmic Maps.} To compute the geodesic and logarithmic map between two points, based on the solution above and Section~\ref{sec:single-log}, one might be tempted to parameterize a curve in the latent spaces of all the individual charts. However, a flow can be arbitrarily bad outside its responsible region, and it is unclear how to obtain the final curve based on the ones induced by the charts. Since a geodesic can naturally transition between charts, a more direct approach is to parameterize a curve across multiple latent spaces while accounting for transitions. However, this leads to a particularly challenging optimization problem.

Instead of the costly direct optimization, our key insight is that multi-chart flows enable us to map a curve in the ambient space onto the manifold through the function $\bx\rightarrow \sum_{c=1,\dots ,C} r_{c}(\bx) \tilde{\bh}_{c}\left(\tilde{\bh}_{i}^{\dagger}\left(\bx\right)\right)$. Using this, we propose to parameterize in the ambient space a curve $\gamma_\phi(t)$, e.g., a spline, and optimize the energy of the mapped curve on the learned manifold. The logarithmic map can then be obtained as the initial velocity of the resulting curve (see Algorithm~\ref{alg:mult-log}).

\begin{algorithm}[tb]
\caption{Algorithm for solving the exponential maps using a fixed $T$.}
\For{$t \gets 0$ \KwTo $T-1$}{
    $[\bc , \br] \gets \text{filter}\left(\text{resps}\left(\bx_{t}\right), \text{resp\_threshold}\right)$ \Comment*[r]{\small \textnormal{// Get filtered responsibilities}}
    \For{$c \in\ \{1,\dots ,C\}$}{
        $[\bx_{t+1}^{c}, \bbv_{t+1}^{c}] \gets \text{Exp}_{c,\Delta t}\left(\bx_{t},\bbv_{t}\right)$ \Comment*[r]{\small \textnormal{// Move on the geodesic induced by the $c$-th flow for $\Delta t$}}
    }
    $\bx_{t+1} \gets \sum_{c=1}^{C} r_{c}\bx_{t+1}^{c}$\;
    $\bbv_{t+1} \gets \sum_{c=1}^{C} r_{c}\bbv_{t+1}^{c}$\;
}
\label{alg:mult-exp-euler}
\end{algorithm}

\begin{algorithm}[tb]
\caption{Algorithm for solving the geodesics and logarithmic maps.}
Initialize $\gamma_{\phi}(t)$\;
\For{$i \gets 1$ \KwTo $N_{\text{iters}}$}{
    \For{$t \gets 0$ \KwTo $T-1$}{
        $[\bbc,\br] \gets \text{filter}\left(\text{resps}\left(\bx_{t}\right), \text{resp\_threshold}\right)$ \Comment*[r]{\small \textnormal{// Get filtered responsibilities}}
        $\bx_{t} \gets \sum_{c=1}^{C}r_{c}(\gamma_{\phi}(t))\tilde{\bh}_{c}\left(\tilde{\bh}_{c}^{\dagger}\left(\gamma_{\phi}(t)\right)\right)$\;
    }
    Minimize $E\left(\left\{\bx_{t},t=1,\dots ,T-1\right\}\right)$ with respect to $\gamma_\phi(t)$
}
Fit a new curve $\tilde{\gamma}_\phi(t)$ to $\sum_{c=1}^{C}r_{c}(\gamma_{\phi}(t))\tilde{\bh}_{c}\left(\tilde{\bh}_{c}^{\dagger}\left(\gamma_{\phi}(t)\right)\right)$\;
Compute $\bbv_{0}$ from $\tilde{\gamma}_{\phi}(t)$ as the initial velocity\;
\label{alg:mult-log}
\end{algorithm}

\begin{figure*}[h]
    \centering
    \includegraphics[width=0.7\linewidth]{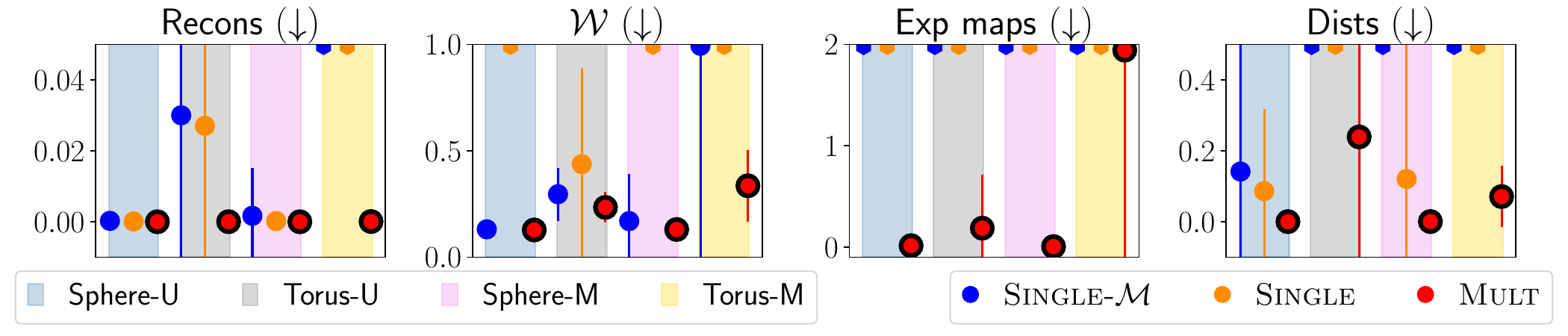}
    \caption{Evaluation metrics on sphere and torus, best methods are circled. \mult\ consistently leads to better performances. Since geodesics are not unique, in some cases, the found one is not the shortest.}
    \label{fig:sphere-torus-results}
\end{figure*}

\begin{figure*}[h]
    \centering    
    \includegraphics[width=0.6\linewidth]{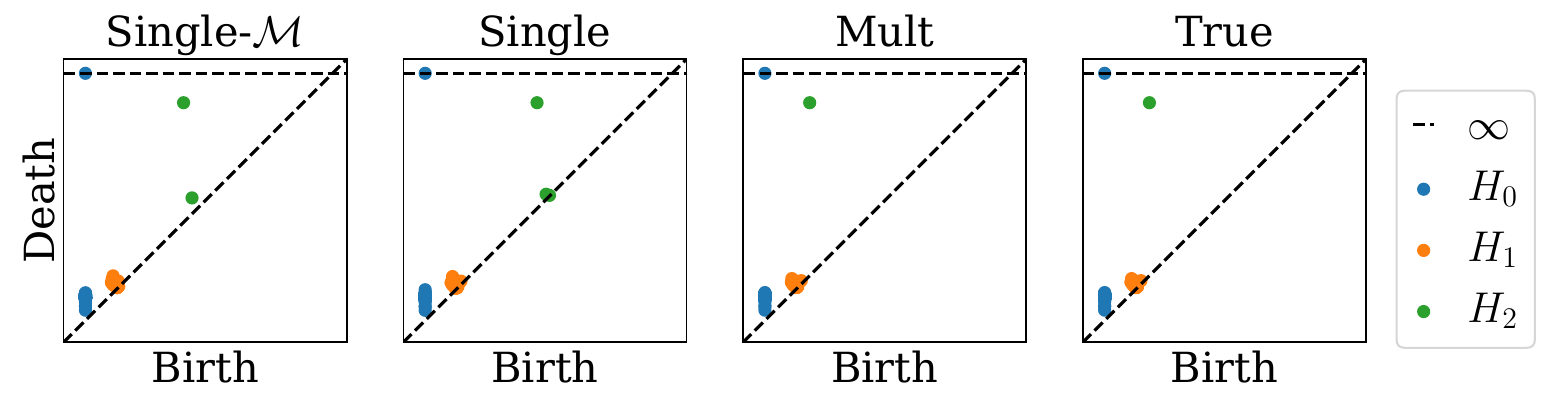}
    \caption{Example persistence diagrams for the Sphere-U data. Points further to the top left corner indicate more significant features. \mult\ provides the most faithful representation of the underlying topology.}
    \label{fig:sphere-torus-diagrams}
\end{figure*}

\section{EXPERIMENTS}

We perform experiments across various datasets to demonstrate the utility of multi-chart flows to model manifold valued data while implicitly capturing the underlying geometry and topology. We compare multi-chart flows (\mult) with single-chart flows trained through both sequential manifold learning and density estimation (\singlem) and direct MLE training (\single). 
For the geodesics, it is known that solutions are not unique; e.g., given two points that are approximately on the opposite of the sphere, the solver may as well find the curve completely opposite of the shortest geodesic, which results in different logarithmic map yet reasonable distance. Hence, we focus on the distances instead of the logarithmic maps.
We evaluate the methods using five complementary metrics: (1) the reconstruction loss measured on a test set, (2) the Wasserstein distance of generated samples to a test set, (3)(4) the mean squared errors of the estimated exponential maps and geodesic distances compared with the ground truths, and (5) the faithfulness of the persistence diagram in capturing the topology. For fair comparisons, we always ensure that all kinds of flows have the same total number of layers for both the outer flows and inner flows. For RealNVP flows the numbers of parameters are exactly the same for single-chart flows and multi-chart flows. For Neural Spline flows since we employ LU Linear Permute before and after every Neural Spline layers, the numbers of parameters differ slightly between single-chart flows and multi-chart flows, but remain comparable.
Further experimental details can be found in Section~\ref{sec:app-exp-details}. We provide code that can be used to reproduce our results at \url{https://github.com/ksnxr/multi-chart-flows}.

\begin{figure*}[h]
    \centering
    \begin{tabular}{cc}
        \includegraphics[width=0.55\textwidth]{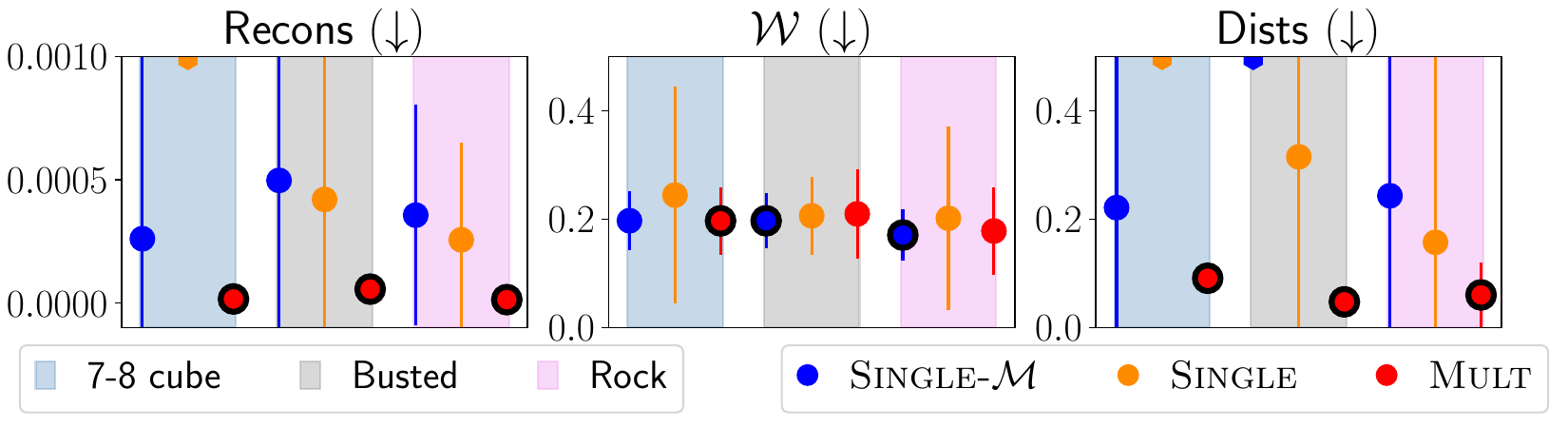} & \includegraphics[width=0.18\textwidth]{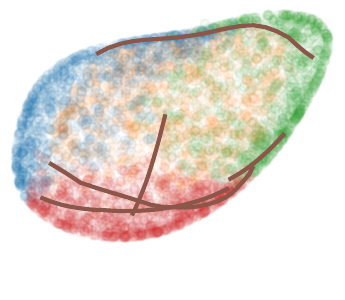}
    \end{tabular}
    \caption{Left: evaluation metrics on triangular meshes, best methods are circled. \mult\ consistently leads to better reconstructions and distance estimates while having competitive sample quality. Right: a trained multi-chart flow enables us to estimate the logarithmic maps.}
    \label{fig:mesh-results}
\end{figure*}

\begin{figure*}[h]
    \centering
    \begin{tabular}{cc}
        \includegraphics[width=0.45\textwidth]{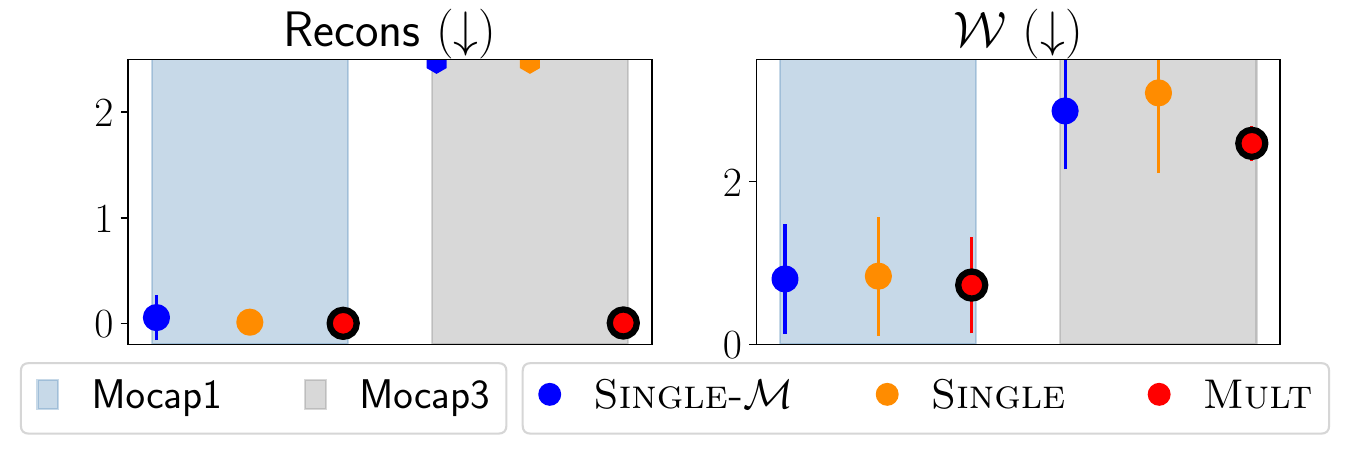} & \includegraphics[width=0.15\textwidth]{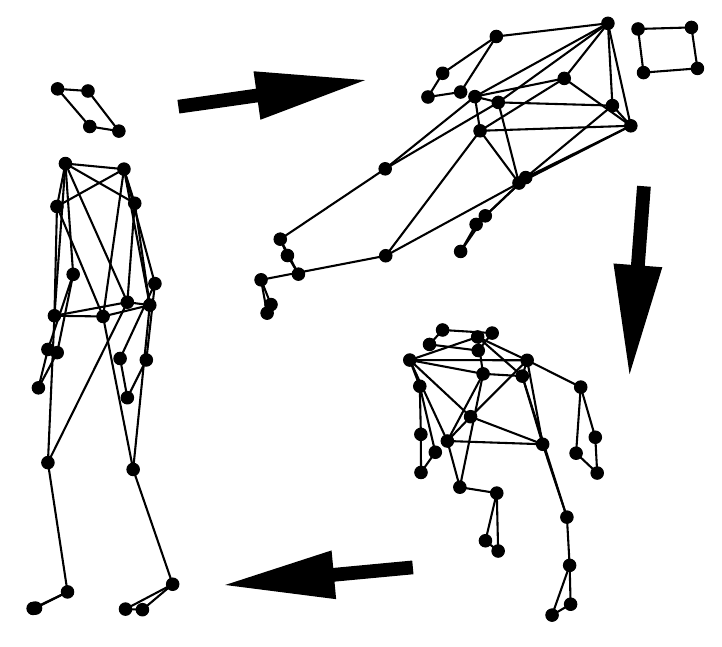}
    \end{tabular}
    \caption{Left: evaluation metrics on Mocap, best methods are circled. \mult\ has the best reconstructions and sample quality for both Mocap1 and Mocap3. Right: The Mocap frame is rotated, here along the fixed axis. The ambient coordinates thus trace along a low dimensional manifold.}
    \label{fig:mocap-results}
\end{figure*}

\begin{table*}[!htbp]
        \begin{center}
        \caption{The lengths of the $5$ segments between $5$ points uniformly placed along the topologically-circular data manifold for each method. Ground truth is approximated using the length of the discretized curve with $1000$ points. Only \mult\ correctly reflects the circular nature of the data. Single-chart flows always show a linear-like structure where one segment is significantly larger.}
        \vspace{1mm}
                \begin{tabular}{llllll}
        			\toprule
        			\textbf{Method} & \multicolumn{5}{c}{\textbf{Lengths} (mean and variance over ten models randomly  initialized)} \\
        			\midrule
        			\singlem & $15.1 \pm 10.6$ & $18.5 \pm 13.6$ & $25.4 \pm 17.4$ & $11.9 \pm 0.2$ & $21.7 \pm 15.9$ \\
        			\single & $18.4 \pm 13.7$ & $25.3 \pm 16.8$ & $11.2 \pm 0.0$ & $15.2 \pm 10.2$ & $21.7 \pm 16.0$ \\
        			\mult & $11.6 \pm 0.0$ & $11.7 \pm 0.0$ & $11.2 \pm 0.0$ & $11.9 \pm 0.2$ & $11.3 \pm 0.0$ \\
        			\textbf{Ground Truth} & 11.6 & 11.7 & 11.2 & 11.8 & 11.3 \\
        			\bottomrule
    		      \end{tabular}
                \label{tbl:mocap-distances}
        \end{center}
\end{table*}

\paragraph{Sphere and Torus.}
\label{sec:sphere-torus}

We consider modeling distributions supported on two dimensional sphere and two dimensional torus using RealNVP flows \citep{dinh2017realnvp}, where the data and the ground truth geometric quantities are obtained using Geomstats \citep{miolane2020geomstats_jmlr,miolane2024geomstats_software} and Pyro \citep{bingham2019pyro}. For \singlem\ and \single, we use $12$ layers for $\bg$ and $36$ layers for $\bh$. For \mult, we use $4$ charts, each chart with $3$ layers for $\bg$ and $9$ layers for $\bh$. For both manifolds, we consider both the uniform distribution, denoted as -U, and the mixture of Von-Mises Fisher and Bivariate Von-Mises distributions, denoted as -M, respectively. 
Figure~\ref{fig:sphere-torus-results} shows that \mult\ yields the best performances in all cases considered. We additionally report the persistence diagrams computed based on the corresponding distance estimates. If these pairwise distances do not accurately reflect the topology of the underlying manifold, the resulting persistent homology may fail to recover the true topological features. Figure~\ref{fig:sphere-torus-diagrams} verifies that the distances as learned using \mult\ allow to correctly identify the topological structure of the sphere; further diagrams can be found in
Section~\ref{sec:app-res-sphere-torus}.

\paragraph{Triangular Meshes.}

We use the same RealNVP flows \citep{dinh2017realnvp} as for sphere and torus
and triangular meshes provided by Trimesh \citep{dawson2025trimesh}. Inspired by Trimesh, we refine the mesh and approximate the ground truth distances between the points using a neighborhood graph; for a dense enough mesh this approximation is reasonable. Figure~\ref{fig:mesh-results} shows that for these irregular shapes \mult\ still provides consistently better reconstructions and distance measures. Also, it enables calculating the exponential maps and logarithmic maps, even if the true ones are intractable.

\paragraph{Motion Capture Data.}

We use the motion capture data from MocapToolbox \citep{burger2013mocap} and generate datasets where the frame is randomly rotated. We consider (1) the axis is fixed and the angle is uniformly at random which forms a 1-dimensional manifold (\textit{Mocap1}), and (2) both the axis and the angle are uniformly at random, which forms a 3-dimensional manifold due to the connection to the well-known $SO(3)$ manifold (\textit{Mocap3}). As the frame is rotated, it traces a trajectory in the embedding space, which naturally inherits the topology of the employed rotations. The embeddings of the manifold may not be isometric and the ground truth exponential maps and logarithmic maps are intractable. While the \textit{ground truth} distances are in principle intractable, we approximate them numerically by calculating the lengths of discretized curves connecting the embeddings using $1000$ points and report them as the ground truth. For a model with correct topology, one would expect the circular topology to be preserved. We employ Neural Spline flows \citep{durkan2019nsf}. For Mocap1, with \mult\ we employ $2$ charts, each with $6$ layers for $\bg$ and $12$ layers for $\bh$; \singlem\ and \single\ use $12$ layers for $\bg$ and $24$ layers for $\bh$. For Mocap3, \mult\ uses $5$ charts, each with $4$ layers for $\bg$ and $8$ layers for $\bh$, while \singlem\ and \single\ use $20$ layers for $\bg$ and $40$ layers for $\bh$.
Figure~\ref{fig:mocap-results} shows that \mult\ always results in the best reconstructions and on-par or better sample quality. Additionally, we analyzed the distances as learned by the models in terms of Mocap1. Table~\ref{tbl:mocap-distances} shows that only \mult\ consistently learn a reasonable measure of lengths, agreeing well with the ground truth and reflecting the correct circular structure of Mocap1.

\section{DISCUSSION}

An important contribution of our paper is to handle non-overlapping charts as they are bound to happen in machine learning applications. We thus developed the first numerical algorithms, to the best of our knowledge, for computing geodesics on such multi-charted manifolds. We also demonstrated that multiple charts are essential when aiming to capture topology. While this is a long-standing knowledge in mathematics, the topic seems to have been underappreciated in the machine learning community.

For generative modeling, one can make arguments that composing multiple smaller models is better than utilizing a single large model \citep{du2024position}. Driven by a different motivation, we make similar findings. It is an interesting question as for whether the success of model composition is in part due to the non-trivial manifold structure of real world datasets.

We focus on normalizing flows which enable fast and exact density evaluations, due to their ability to achieve fast inference and efficient computations of geometric quantities involving the Jacobians, along with the tractability of training degenerate normalizing flows. Flow matching \citep{lipman2023fm} is a recent family of algorithms that learns a continuous normalizing flow parameterized by a vector field, without needing numerical integrations during training. It has been extended to Riemannian manifolds \citep{chen2024rfm}, however both the base distribution and the velocities are defined on the target manifold itself. Free-form flows \citep{draxler2024fff,sorrenson2024lifting} is another family of normalizing flow method that has been extended to Riemannian manifolds. However, they are not guaranteed to be bijections, making the differential geometric interpretations less grounded. While we assume that the dimensionality is known beforehand, \citet{Zhang2023spread_flows} shows that is possible to identify the intrinsic dimensionality while training the flow. We leave further explorations utilizing these methods as future work.

\paragraph{Limitations.} Multi-chart flows are more technically involved and can be difficult to train especially for manifolds with high curvature. Similarly, computing geodesics becomes more challenging as curvature increases. However, our proposed methodology is, in principle, generic as more advanced flow models can be easily incorporated to improve the performance.

\paragraph{Conclusion.} Single-chart flows are fundamentally incapable of respecting complex geometric structures. In contrast, we showed that this is possible when using multi-chart flows which can capture the geometry and topology of the underlying space while learning the distribution. In addition, we provided specifically designed numerical methods to  approximate geometric quantities, such as geodesics, on the learned manifolds induced by the flows that may not perfectly align. Overall, our work enables computational differential geometry on data manifolds with non-trivial topology.

\subsubsection*{Acknowledgements}
HY, MH and AK were supported by the Research Council of Finland Flagship programme: Finnish Center for Artificial Intelligence FCAI and by the grants 363317 and 348952, and acknowledge the research environment provided by ELLIS Institute Finland. SH was supported by a research grant (42062) from VILLUM FONDEN, received funding from the European Research Council (ERC) under the European Union’s Horizon research and innovation programme (grant agreement 101125993) and was partly funded by the Novo Nordisk Foundation through the Center for Basic Machine Learning Research in Life Science (NNF20OC0062606). GA was supported by the DFF Sapere Aude Starting Grant ``GADL''. The authors wish to acknowledge CSC - IT Center for Science, Finland, for computational resources.

\bibliography{paper_references}

\section*{Checklist}

\begin{enumerate}

  \item For all models and algorithms presented, check if you include:
  \begin{enumerate}
    \item A clear description of the mathematical setting, assumptions, algorithm, and/or model. [Yes]
    \item An analysis of the properties and complexity (time, space, sample size) of any algorithm. [Yes]
    \item (Optional) Anonymized source code, with specification of all dependencies, including external libraries. [Yes]
  \end{enumerate}

  \item For any theoretical claim, check if you include:
  \begin{enumerate}
    \item Statements of the full set of assumptions of all theoretical results. [Yes]
    \item Complete proofs of all theoretical results. [Yes]
    \item Clear explanations of any assumptions. [Yes]     
  \end{enumerate}

  \item For all figures and tables that present empirical results, check if you include:
  \begin{enumerate}
    \item The code, data, and instructions needed to reproduce the main experimental results (either in the supplemental material or as a URL). [Yes]
    \item All the training details (e.g., data splits, hyperparameters, how they were chosen). [Yes]
    \item A clear definition of the specific measure or statistics and error bars (e.g., with respect to the random seed after running experiments multiple times). [Yes]
    \item A description of the computing infrastructure used. (e.g., type of GPUs, internal cluster, or cloud provider). [Yes]
  \end{enumerate}

  \item If you are using existing assets (e.g., code, data, models) or curating/releasing new assets, check if you include:
  \begin{enumerate}
    \item Citations of the creator If your work uses existing assets. [Yes]
    \item The license information of the assets, if applicable. [Yes]
    \item New assets either in the supplemental material or as a URL, if applicable. [Yes]
    \item Information about consent from data providers/curators. [Not Applicable]
    \item Discussion of sensible content if applicable, e.g., personally identifiable information or offensive content. [Not Applicable]
  \end{enumerate}

  \item If you used crowdsourcing or conducted research with human subjects, check if you include:
  \begin{enumerate}
    \item The full text of instructions given to participants and screenshots. [Not Applicable]
    \item Descriptions of potential participant risks, with links to Institutional Review Board (IRB) approvals if applicable. [Not Applicable]
    \item The estimated hourly wage paid to participants and the total amount spent on participant compensation. [Not Applicable]
  \end{enumerate}

\end{enumerate}

\clearpage
\appendix

\onecolumn
\aistatstitle{Learning Geometry and Topology via Multi-Chart Flows: \\
Supplementary Materials}

\section{PROOF OF THEOREM 1}
\label{sec:app-exp-maps-proof}

We restate Theorem 1, and provide the proof.

\begin{theorem}
The geodesic equation in the latent space induced by the pullback metric is given by
\begin{equation}
\ddot{z}(t)^{k} = - g^{kl} \sum_{i,j,m} \frac{\partial^{2}x_{m}}{\partial z_{i}\partial z_{j}} \frac{\partial x_{m}}{\partial z_{l}} \dot{z}^{i} \dot{z}^{j}.
\end{equation}
\label{thm:single-geodesic}
\end{theorem}

\begin{proof}
Recall that the pullback metric is given by
\begin{align}
\bg &= \left(\frac{\partial\bx}{\partial\bz}\right)^{\top} \frac{\partial\bx}{\partial\bz},\\
g_{ij} &= \sum_{m}\frac{\partial x_{m}}{\partial z_{i}} \frac{\partial x_{m}}{\partial z_{j}}.
\end{align}

The Christoffel symbols and the geodesic equation are given by \citep{Lee2018rm} (Equation 5.8 and Equation 4.16)
\begin{align}
\Gamma^{k}_{ij} &= \frac{1}{2}g^{kl}\left(\partial_{i}g_{jl} + \partial_{j}g_{il} - \partial_{l}g_{ij}\right),\\
\ddot{z}(t)^{k} + \dot{z}^{i}(t) \dot{z}^{j}(t) \Gamma^{k}_{ij}(z(t)) &= 0.
\end{align}
We have
\begin{equation}
\partial_{l}g_{ij} = \sum_{m} \frac{\partial^{2}x_{m}}{\partial z_{i}\partial z_{l}} \frac{\partial x_{m}}{\partial z_{j}} + \frac{\partial x_{m}}{\partial z_{i}} \frac{\partial^{2}x_{m}}{\partial z_{j}\partial z_{l}},
\end{equation}
and
\begin{align}
&\quad \frac{1}{2}\left(\partial_{i}g_{jl} + \partial_{j}g_{il} - \partial_{l}g_{ij}\right) \\
&= \frac{1}{2}\left(\sum_{m} \frac{\partial^{2}x_{m}}{\partial z_{j}\partial z_{i}} \frac{\partial x_{m}}{\partial z_{l}} + \frac{\partial x_{m}}{\partial z_{j}} \frac{\partial^{2}x_{m}}{\partial z_{l}\partial z_{i}} + \frac{\partial^{2}x_{m}}{\partial z_{i}\partial z_{j}} \frac{\partial x_{m}}{\partial z_{l}} + \frac{\partial x_{m}}{\partial z_{i}} \frac{\partial^{2}x_{m}}{\partial z_{l}\partial z_{j}}\right. \\
&\quad \left. - \frac{\partial^{2}x_{m}}{\partial z_{i}\partial z_{l}} \frac{\partial x_{m}}{\partial z_{j}} - \frac{\partial x_{m}}{\partial z_{i}} \frac{\partial^{2}x_{m}}{\partial z_{j}\partial z_{l}} \right) \\
&= \frac{1}{2}\left(\sum_{m} 2\frac{\partial^{2}x_{m}}{\partial z_{i}\partial z_{j}} \frac{\partial x_{m}}{\partial z_{l}} \right) = \sum_{m} \frac{\partial^{2}x_{m}}{\partial z_{i}\partial z_{j}} \frac{\partial x_{m}}{\partial z_{l}}.
\end{align}
As such,
\begin{align}
&\quad \dot{z}^{i}(t)\dot{z}^{j}(t)\Gamma^{k}_{ij}(z(t)) = \dot{z}^{i}(t)\dot{z}^{j}(t) \frac{1}{2}g^{kl}\left(\partial_{i}g_{jl} + \partial_{j}g_{il} - \partial_{l}g_{ij}\right) \\
&= \dot{z}^{i}(t)\dot{z}^{j}(t) g^{kl} \sum_{m} \frac{\partial^{2}x_{m}}{\partial z_{i}\partial z_{j}} \frac{\partial x_{m}}{\partial z_{l}} = g^{kl}\sum_{m} \frac{\partial^{2}x_{m}}{\partial z_{i}\partial z_{j}} \frac{\partial x_{m}}{\partial z_{l}} \dot{z}^{i} \dot{z}^{j},
\end{align}
and
\begin{equation}
\ddot{z}(t)^{k} = - g^{kl} \sum_{i,j,m} \frac{\partial^{2}x_{m}}{\partial z_{i}\partial z_{j}} \frac{\partial x_{m}}{\partial z_{l}} \dot{z}^{i} \dot{z}^{j}.
\end{equation}
\end{proof}

The above expression explicitly gives the form of the geodesic accelerations and affords tractable implementations using automatic-differentiation systems, as discussed in detail in Section~\ref{sec:app_imp_geo}.

\section{EM TRAINING}
\label{sec:app-em}

We briefly explain the entire EM procedure, using as references \citet{murphy2023pml} and \citet{bishop2006prml}. We provide an outline of the general procedure of using the EM algorithm to train the model, and provide the expressions for the precise EM training procedure in our case.

Denote $f(c_{i})$ as an arbitrary set of distributions over $c$ for each $\bx_{i}$. One can write
\begin{equation}
\sum_{i=1}^{N}\log q(\bx_{i}) = \sum_{i=1}^{N}\log\left(\sum_{c_{i}} q(\bx_{i},c_{i})\right) = \sum_{i=1}^{N}\log\left(\sum_{c_{i}} f(c_{i})\frac{q(\bx_{i},c_{i})}{f(c_{i})}\right).
\end{equation}

Using Jensen's inequality, we have
\begin{equation}
\sum_{i=1}^{N}\log(\sum_{c_{i}} f(c_{i})\frac{q(\bx_{i},c_{i})}{f(c_{i})}) \geq \sum_{i=1}^{N}\sum_{c_{i}} f(c_{i}) \log\left(\frac{q(\bx_{i},c_{i})}{f(c_{i})}\right),
\end{equation}
which is the Evidence Lower BOund (ELBO). In order to achieve maximum likelihood training, we would like to maximize the above quantity.

For the E step, observe that
\begin{align}
&\quad \sum_{c_{i}} f(c_{i})\log\left(\frac{q(\bx_{i},c_{i})}{f(c_{i})}\right)= \sum_{c_{i}} f(c_{i})\left(\log \frac{q(c_{i}|\bx_{i})q(\bx_{i})}{f(c_{i})}\right)\\
&= \sum_{c_{i}} f(c_{i})\left(\log \frac{q(c_{i}|\bx_{i})}{f(c_{i})}\right) + \sum_{c_{i}} f(c_{i})\log q(\bx_{i})
= -\text{KL}\left(f(c_{i})\vert q(c_{i}|\bx_{n})\right) + \log q(\bx_{i}).
\end{align}
As such, we can set $f(c_{i})$ to $q(c_{i}|\bx_{i})$ to maximize it. Given $q(c_{i})$ and $q(\bx_{i}|c_{i})$, $q(c_{i}|\bx_{i})$ can be obtained in closed-form using the Bayes rule as
\begin{equation}
q(c_{i}|\bx_{i}) = \frac{q(\bx_{i}|c_{i})q(c_{i})}{q(\bx_{i})} = \frac{q(\bx_{i}|c_{i})q(c_{i})}{\sum_{c=1}^{C} q(\bx_{i}|c)q(c)}.
\end{equation}
Note that $f(c_{i})$ thus has an intuitive interpretation of \textit{responsibilities} of the charts on the point $\bx_{i}$. We thus refer to it accordingly, and denote it as $r_{c}(\bx_{i})$.

For the M step, we use $r_{c}(\bx_{i})$ as derived in the E step and optimize for the best $q(\bx_{i}|c)$ (and possibly $q(c)$). In other words, we now need to maximize the following
\begin{equation}
\sum_{i=1}^{N}\sum_{c=1}^{C} r_{c}(\bx_{i})(\bx_{i}) \log\left(q(\bx_{i},c)\right) = \sum_{i=1}^{N}\sum_{c=1}^{C}r_{c}(\bx_{i})\log q(\bx_{i}|c) + \sum_{i=1}^{N}\sum_{c=1}^{C}r_{c}(\bx_{i})\log q(c).
\end{equation}
Observe that the optimization for $q(c)$ can be carried out in closed-form, while the optimization for $q(\bx_{i}|c)$ affords standard gradient based training.

In the specific case of our multi-chart flows, we directly fix $q(c)$ to be uniform over the charts, and train the individual normalizing flows to model $q(\bx_{i}|c)$. However, since we would like to enforce the flows to perform reconstructions as well, following \citet{Caterini2021rectangular} we additionally add a term penalizing the reconstruction error, resulting in
\begin{equation}
\log \tilde{q}_{\btheta}(\bx_{i}|c) = \log q_{\btheta}(\bx_{i}|c) + \lambda \norm{\bx_{i} - \tilde{\bh_{c}} \left(\tilde{\bh_{c}}^{\dagger}\left(\bx_{i}\right)\right)}^{2}.
\end{equation}
As noted by \citet{Brehmer2020mflows}, depending on the problem, one can add additional terms, e.g. a term that penalizes the hidden variables $\tilde{\bh_{c}}^{\dagger}(\bx_{i})$.

To summarize, when training the mixture model using EM, we alternate between the E step, which sets $r_{c}(\bx_{i})$ based on the flows' log-density evaluations, and the M step, which trains each individual flow based on its responsibilities of the data points.

Additionally, sometimes it is beneficial to include a term penalizing the deviation of the mean responsibilities of the charts averaged over the data points. Since we set $q(c)$ to be uniform over the charts, we expect the mean responsibilities to be close to uniform as well especially when the batch size is larger. We therefore add a loss term in the form of
\begin{equation}
\lambda \sum_{c=1}^{C} \left( \frac{\sum_{i=1}^{N}r_{i}}{N} - \frac{1}{N} \right)^{2}.
\label{eq:resp-regularizer}
\end{equation}
We remark that it has been demonstrated useful to add regularization terms in EM in general \citep{li2005regularized,houdouin2023regularized}.

\subsection{EM and direct MLE}

Both EM and direct MLE are in principle applicable for training multi-chart flows, and we could also focus on MLE and it works roughly as well. We do not see the specific choice of the training strategy as a main contribution or a key choice, which is exactly why we described both alternatives. In our work, we prefer EM primarily because it directly provides the responsibilities, an interpretable quantity that can be be used to monitor training and develop practical algorithms for solving for the exponential maps and logarithmic maps. Furthermore, it enables the regularizer given by Equation~\eqref{eq:resp-regularizer}. This regularizer encourages the different charts to contribute equally to the data points, and helps to prevent one chart having essentially zero responsible data points, a pathology that may appear when training unregularized models. We provide ablation studies on models trained using EM and direct MLE in Section~\ref{sec:app-em-mle-exps}, verifying that they yield similar performances.

\section{GEOMETRY}

\subsection{Explanations Concerning the Geometry}

Consider a $d$ dimensional surface $\bS$ that is parameterized by a map $\bbf: \bZ \subset \R^{d} \rightarrow \R^{D}$, where $D$ is the ambient dimension and $\bZ$ is the parameter space. Under proper smoothness conditions, the pullback Riemannian metric $\bG(\bz) = \bJ_{\bbf(\bz)}^{\top} \bJ_{\bbf(\bz)} \in \R^{d \times d}_{\succ 0}$ is a positive definite matrix that captures the intrinsic geometry of the surface in the parameter space. For example, it is possible to compute shortest paths in $\bZ$ while respecting the geometry of $\bS$. When the topology of $\bS$ is more intrincate, i.e. not homeomorphic to Euclidean, we need multiple maps, where each of them parameterizes a neighborhood on the surface. Similarly, in our model, each normalizing flow induces a Riemannian metric in the associated latent space.

\subsection{Details on the Geometric Operations}
\label{sec:app-geo_details}

Consider $\tilde{\bh}$, $\tilde{\bh}^{\dagger}$, $\bx$ and $\bz$ as defined in the main paper. For points that lie precisely on the learned manifold, $\tilde{\bh}$ and $\tilde{\bh}^{\dagger}$ are a pair of bijective functions that are inverses of each other.

Given a value of $\bz$, one can obtain the unique corresponding $\bx = \tilde{\bh}(\bz)$. The pullback metric can also be tractably evaluated as $\bJ_{\tilde{\bh}}^{\top}\bJ_{\tilde{\bh}}$. Given $\bbv_{\bz}$, one can also obtain the corresponding $\bbv_{\bx}$ using the formula $\bbv_{\bx} = \frac{\partial\tilde{\bh}}{\partial\bz}\bbv_{\bz}$. One way to see that is to note that $\bbv_{\bx} = \frac{\partial \bx}{\partial t} = \frac{\partial \bx}{\partial \bz} \frac{\partial \bz}{\partial t} = \frac{\partial \bx}{\partial \bz} \bbv_{\bz}$. These operations are generally well-defined.

However, due to the degeneracy, multiple $\bx$ can correspond to the same $\bz$. $\tilde{\bh}^{\dagger}$ acts as a projection operation that projects data points onto the manifold, thus naturally induces quotient structures. When $\bx$ lies exactly on the learned manifold and $\bbv_{\bx}$ lies exactly on the tangent space, one can obtain the corresponding $\bbv_{\bz}$ using $\bbv_{\bz} = \frac{\partial\tilde{\bh}^{\dagger}}{\partial\bx}\bbv_{\bx}$. Due to the way that $\tilde{\bh}^{\dagger}$ is composed, we have
\begin{equation}
\frac{\partial\tilde{\bh}^{\dagger}}{\partial\bx} = \frac{\partial\left(\text{Proj}\cdot \bh^{-1}\right)}{\partial\bx} = \frac{\partial\text{Proj}}{\partial\bh^{-1}(\bx)} \frac{\partial \bh^{-1}}{\partial\bx}.
\end{equation}
Projection is a linear operation by definition, and its Jacobian matrix is simply given by a rectangular matrix with ones at the entries where the coordinates are retained. As such, when two distinct values $\bx_{1}$ and $\bx_{2}$ correspond to the same $\bz$, the Jacobian matrices are the same when the corresponding rows of $\frac{\partial\bh^{-1}}{\partial\bx_{1}}$ and $\frac{\partial\bh^{-1}}{\partial\bx_{2}}$ are the same. 

In the case of \mult, the above formulas also enable the definition of a transition operator that moves between charts. Specifically, given the latent coordinates of a point in the $i$th chart $\bz_{i}$, its corresponding representation in the $j$th chart is given by $\bz_{j} = \tilde{\bh}^{\dagger}_{j} \cdot \tilde{\bh}_{i}(\bz_{i})$.

\subsection{Implementing the Geodesic Equation}
\label{sec:app_imp_geo}

Here we discuss how to implement an efficient and numerically stable version of the geodesic equation using modern automatic-differentiation (AD) systems.

Recall the form of the geodesic equation; at each step we need to compute $- g^{kl} \sum_{i,j,m} \frac{\partial^{2}x_{m}}{\partial z_{i}\partial z_{j}} \frac{\partial x_{m}}{\partial z_{l}} \dot{z}^{i} \dot{z}^{j}$. One can divide it into three parts: the inverse of the Riemannian metric $g^{kl}$, the Jacobian $\frac{\partial x_{m}}{\partial z_{l}}$ and the quadratic form $\frac{\partial^{2}x_{m}}{\partial z_{i}\partial z_{j}} \dot{z}^{i} \dot{z}^{j}$. Noting that the Riemannian metric is precisely given by the outer product of the Jacobian, we only need to obtain the Jacobian and the quadratic form using AD.

\paragraph{Jacobian.} In general, for a function $\bbf: \R^{n} \rightarrow \R^{m}$, the cost to obtain its Jacobian is $k n$ for forward mode AD and $k m$ for reverse mode AD, where k is some constant dependent on the function \citep{baydin2018ad}. In our case, we need to obtain the Jacobian of $\bh$, which maps from $\R^{d}$ to $\R^{D}$, with $d < D$. As such, one may prefer forward mode AD. In practice, the batched Jacobian matrices can be implemented by a vectorization of the function that obtains the individual Jacobian over the batch, where for each sample we obtain the Jacobian matrix using forward mode AD.

\paragraph{Quadratic Form.} 

Forward mode AD step efficiently calculates Jacobian-vector product \citep{baydin2018ad}. As such, one can first calculate $\frac{\partial x_{m}}{\partial z_{i}}\dot{z}^{i}$, and then calculate $\frac{\partial^{2}x_{m}}{\partial z_{i}\partial z_{j}} \dot{z}^{i} \dot{z}^{j}$, employing Jacobian-vector product in each. This is to some extent known in the AD community; see e.g. \url{https://github.com/jax-ml/jax/discussions/8456#discussioncomment-1586157}.

We remark that we do not claim to have the most efficient implementation, as the focus of the paper is on obtaining the correct geometric quantities instead of having the fastest speed in doing so.

\subsection{Complexities of Solving for Geometric Quantities}

\paragraph{Exponential Maps}
For \singlem\ and \single, solving the exponential maps directly corresponds to solving for the geodesic equation as an initial value problem. In general, this operation has complexity $O(d^{3})$, due to the matrix inversion, and each step involves computing the Jacobian matrices of the flows. However, $d$ is generally small, and the Jacobians can be computed reasonably efficiently, e.g. in PyTorch \citep{ansel2024pytorch}, via vectorized operations.

A crude analysis of the computational complexity of \mult\ roughly involves solving $C$ such exponential maps for \single\ and \single, since it needs to average over the predictions generated by the individual charts. However, for each position the number of responsible charts can be small in practice, and it is largely efficient.

\paragraph{Geodesics and Logarithmic Maps}
For \singlem\ and \single, obtaining the geodesics and logarithmic maps amounts to optimizing the spline in $d$ dimensional latent space, where in each step the flows need to map points from the latent space to the ambient space.

\mult\ involves more computations, due to the need of optimizing a spline in the ambient $D$ dimensional space. In each step one needs to reconstruct the spline using the model making use of responsibilities.

\subsection{Connections to Gaussian Mixture Models (GMMs)}

A Gaussian Mixture Model (GMM) constructs a mixture of Gaussian distributions, which can be expressive when the number of mixture components is large enough. The mixture formulation is similar to our proposed mixture of flows formulation. However, in our case each individual component is a normalizing flow, which is naturally more flexible than a Gaussian distribution. In general, in our scenario there are two ways to employ a GMM for our purpose, 1. using a GMM in the ambient space and 2. using a GMM directly on the Riemannian manifold.

For the first interpretation, a GMM approximates the data density by fitting multivariate normal distributions to the data, and it is ideal if the data indeed follows such a structure. However, when data possesses a nonlinear manifold structure, a standard GMM is suboptimal. For instance, when data is uniformly distributed on a sphere, the GMM may have trouble capturing the nonlinear structure of the underlying manifold, even when the number of mixture components is large. First, since the dimensionality of the data manifold is $2$, the covariances of each individual component must be low-rank. Second, our main goal is not to approximate the data density, but instead to learn a parameterization of the underlying Riemannian manifold. We can think of each component of the GMM as a linear map $\bbf_{i}(\bx) = \bmu_{i} + \bA_{i}\bz$, where $\bA_{i}$ is the decomposition of the covariance $\boldsymbol{\Sigma}_{i}$ and $\bz \sim \mathcal{N}(\mathbf{0}, \boldsymbol{I}_{2})$. Clearly, this parameterization is not able to capture the nonlinear geometry of the sphere.

Instead, when using our approach on the same example, motivated by differential geometry, local neighborhoods of the sphere are parameterized using nonlinear bijective functions $\bbf_{i}: \R^{2} \rightarrow \R^{3}$. In order to learn such local parameterizations from data, we leverage rectangular normalizing flows. Unlike Gaussians, these functions can learn the nonlinear structure of the underying manifold. In order to learn the functions, we rely on a probabilistic approach because this allows the model to learn how to separate the local neighborhoods without explicit supervision, through adapting the responsibilities of the mixture model.

To summarize, our goal is to learn a collection of nonlinear bijections $\bbf_{i}: \R^{d} \rightarrow \R^{D}$, where each of them parameterizes a local neighborhood of a nonlinear manifold which resembles the data. We utilize a mixture model to enable the functions to learn the neighborhoods in an unsupervised fashion based on the responsibilities. This can be interpreted as similar to a GMM, where the multivariate normal distributions are replaced by normalizing flows. However, we want to highlight that our main goal is to learn the parameterization of the underlying manifold, instead of performing data density estimation in the ambient space.

In terms of the second interpretation, our method is similar to a GMM defined on a given Riemannian manifold. Indeed, there are extensions of the classic GMMs to nonlinear manifolds, using as building blocks the extensions of Gaussian distributions to Riemannian manifolds \citep{pennec2006intrinsic}. However, in these approaches the underlying Riemannian structure is assumed known beforehand, e.g. we know the data lies on a sphere. Instead, our work focuses on actually learning the unknown underlying Riemmannian manifold from data. With our approach, it is possible to learn the manifold structure, and we also provide the numerical tools that enable one to perform downstream tasks, e.g. statistical analysis. Note that normalizing flows are capable of falling back to GMMs when the target distribution is indeed given by a GMM.

\subsection{Connections to Resampling Base Distributions}

\citet{stimper2022resampling} proposed changing the latent prior while using a single normalizing flow to model the data. This implies that the latent space must have the same dimensionality as the data space, as such, we cannot cover the data manifold in a differential geometric sense.

For example, consider data lying on a sphere. The learned latent prior could also be a sphere, which captures the topology of the data manifold and implies that the trainded normalizing flow will approximately be an identity mapping. As such, the flow will not induce a useful pullback metric for computing geodesics. It might be that we can compute paths that follow high density regions, but these are not necessarily geodesics. In addition, we cannot obtain a $2$ dimensional latent space.

To summarize, we believe that \citet{stimper2022resampling} can serve as a surrogate model or as a pre-processing step for our multi-chart flows formulation.

\section{EXPERIMENTAL DETAILS}
\label{sec:app-exp-details}

For open source libraries we provide the license information inside brackets after the citations.

We mainly use PyTorch \citep{ansel2024pytorch} (custom license), NumPy \citep{harris2020numpy} (custom license) and SciPy \citep{virtanen2020scipy} (BSD-3 license) for the experiments, with the normalizing flows implementations largely based on \citet{stimper2023normflows} (MIT license). We use Scikit-TDA \citep{nathaniel2019scikit-tda} (MIT license) to plot the persistence diagrams.

For all models, we tune the learning rates among $[1e-4,3e-4,1e-3]$, and tune the reconstruction factors among $[100, 1000, 10000]$. We tune the hyperparameters using a single run for each configuration with seed $1$, then for the best configuration run two additional runs using as seeds $2$ and $3$ and report the results obtained using all three runs. For \singlem, in the first half of all epochs the model is trained solely based on reconstruction errors as induced by $\bh$, while in the second half of all epochs the model is trained solely based on log-probabilities as induced by $\bg$.

For \single\ and \mult, we use early stopping solely in terms of the reconstruction loss; for \singlem, in the manifold learning phase we use early stopping in terms of the reconstruction loss, while in the density estimation phase we use early stopping in terms of $\W$. This may lead to some advantages for \singlem\ in terms of $\W$.

For sphere, torus and triangular meshes experiments, we use RealNVP flows \citep{dinh2017realnvp}, where the MLPs use Tanh activation, and \textit{double} data type. We remark that it is important to use a smooth activation function when intending to solve the IVP, due to the need to differentiate through the flow. We sample a total of $12000$ data points, and use $25\%$ as the validation set and the remaining as the train set. We train all models for a maximum of $1000$ epochs, using a batch size of $256$. For \mult, $200$ epochs are used for pretraining. The pretraining datasets for the charts are based on K-Means clusters obtained through Scikit-Learn \citep{pedregosa2011sklearn} (BSD-3-Clause license). We use early stopping with patience $50$, starting validation after the model has been trained for $100$ epochs in the current stage. We use some separate $10000$ data points as the test set.

For \single\ and \mult, we use as validation loss the reconstruction loss on the val set. For \singlem, we use sequential training, where the first phase uses as validation loss the reconstruction loss, and the second phase uses as validation loss the Wasserstein distance.

For solving the exponential maps, most numerical integrations are carried out using \textit{scipy.integrate.solve\_ivp} from SciPy, where we use the default Dopri-$5$ solver \citep{dormand1980dopri,shampine1986practical} with $rtol=1e-3$ and $atol=1e-6$. For solving the logarithmic maps, our implementation is largely based on Stochman \citep{detlefsen2021stochman} (Apache-2.0 license), where we optimize a parameterized spline using gradient descent.

We use Python Optimal Transport \citep{flamary2021pot,flamary2024pot} (MIT license) to calculate the Wasserstein distances between batches of data and batches of samples drawn using the models, where the subsamplings are performed for $5$ times and each batch has $1024$ samples.

\subsection{Error Bars}

For all visualizations where an error bar is shown, the bar is plotted as mean $\pm$ $3$ times the standard deviation.

The reported means and standard deviations are obtained using multiple runs with different random seeds. Experimental results as reported in the main paper use $10$ independent runs unless otherwise stated, while the experimental results as reported in the Appendix generally use $3$ independent runs.

\subsection{Sphere and Torus}
\label{sec:app-sphere-torus}

We consider four distributions on the two dimensional sphere and the two dimensional torus. Specifically, they are: uniform distribution on the sphere (Sphere-U), mixture of von-Mises Fisher (vMF) distribution on the sphere (Sphere-M), uniform distribution on the torus (Torus-U) and mixture of Bivariate von-Mishes (BvM) distribution on the torus (Torus-M).

The datasets are generated mainly using as tools Geomstats \citep{miolane2020geomstats_jmlr,miolane2024geomstats_software} (MIT license) and Pyro \citep{bingham2019pyro} (Apache-2.0 license).

For the mixture of vMF distribution, we use four components, with the means given by $[\frac{1}{\sqrt{3}}, \frac{1}{\sqrt{3}}, \frac{1}{\sqrt{3}}], [-\frac{1}{\sqrt{3}}, -\frac{1}{\sqrt{3}}, \frac{1.0}{\sqrt{3}}], [-\frac{1}{\sqrt{3}}, \frac{1}{\sqrt{3}}, -\frac{1}{\sqrt{3}}], [\frac{1}{\sqrt{3}}, -\frac{1}{\sqrt{3}}, -\frac{1}{\sqrt{3}}]$ and $\kappa=5$. 

For the mixture of BvM distribution, we use four components specified in angular coordinates. The means are given by $[0, 0], [\pi, 0], [0, \pi], [\pi, \pi]$. Each component has concentration $1$, with the correlation between the two dimensions $0$.

When performing evaluations of the exponential maps, the logarithmic maps and distances, based on preliminary experiments, we observe that it is beneficial to choose the evaluation points as distant to each other as possible to make the resulting persistence diagrams clear. As such, on the sphere, we use Fibonacci lattice as implemented by \citet{brinkman2025fiblat} (MIT license). On the torus, since it is the product manifold formed by two circles, we analytically draw samples on the circles using angular coordinates and form the final samples using Cartesian products. For exponential maps, we use $100$ data points, resulting in $4950$ evaluations, where we solve $128$ maps in a batch. For logarithmic maps and distances we use $225$ data points, resulting in $25200$ evaluations, where we solve $256$ maps in a batch. Note that the ground truth distance of $\bx_{1}$ to $\bx_{2}$ is naturally the same as the ground truth distance of $\bx_{2}$ to $\bx_{1}$, and the ground truth distance of $\bx_{1}$ to itself is naturally $0$. The exponential maps and logarithmic maps are evaluated based on the pairs determined by \textit{tril\_indices} function in NumPy and PyTorch with offset $-1$.

\subsection{Triangular Meshes}

The triangular meshes experiments are largely based on Trimesh \citep{dawson2025trimesh} (MIT license). We use \textit{7\_8ths\_cube.stl}, referred to as $7$-$8$ cube, \textit{busted.STL}, referred to as Busted, and \textit{rock.obj.bz2}, referred to as Rock, as provided by Trimesh \citep{dawson2025trimesh}, while refining the meshes partly to make them smooth and the resulting distance approximations reasonable. Additionally, the meshes are normalized.

For the evaluation of distances, we use $100$ data points. Following the reasoning in Section~\ref{sec:app-sphere-torus}, there are a total of $4950$ distance evaluations to be carried out, with each batch solving $256$ evaluations. Due to the lack of a principled approach to draw maximally distant samples on the meshes, the data points are simply drawn uniformly. Following \citet{dawson2025trimesh}, each \textit{ground truth} distance between two points is approximated as the shortest distance between the closest vertices to the points on the graph induced by the mesh and solved using NetworkX \citep{hagberg2008networkx} (3-clause BSD license).

\subsection{Motion Capture Data}

We use the initial frame from \textit{dance1} as provided by MocapToolbox \citep{burger2013mocap} (GNU general public license). We construct two variants of Mocap datasets, \textit{Mocap1} and \textit{Mocap3}.

In Mocap1, the initial frame is randomly rotated around the axis given by $[1 / \sqrt{14}, 2/\sqrt{14}, 3/\sqrt{14}]$, with the rotation angle uniformly sampled between $0$ and $2\pi$. We use $30000$ points as the train set, $10000$ points as the val set and $10000$ points as the test set. In Mocap3, the initial frame is rotated using a random element from SO(3). We use $50000$ points as the train set, $10000$ points as the val set and $10000$ points as the test set. For both datasets, we standardize the data based on the train set. In Mocap1 the batch size is $256$, while for Mocap3 in preliminary experiments we observe that using a larger batch size stabilizes training especially for \mult, so we employ a batch size of $1024$.

We use Neural Spline flows \citep{durkan2019nsf} with ReLU activation and \textit{float} data type. For pretraining as in \mult, we use Constrained K-Means clustering \citep{bennett2000constrained} as implemented in \citet{levy-kramer2018k-means-constrained} (BSD-3-Clause license) to obtain the clusters, where the clusters are constrained to be between $0.9t$ and $1.1t$, where $t$ is the target size given by dividing the total number of samples by the number of charts. We use early stopping with infinite patience and perform validations throughout the process. For Mocap1, \single\ is trained for the first $20$ epochs solely based on reconstructions, while during pretraining of \mult\ which has a total of $30$ epochs the first $20$ epochs are solely based on reconstructions and the following $10$ epochs also account for log-probabilities. For Mocap3, the first $50$ epochs of \single\ are solely based on reconstructions, while for \mult\, among the $70$ epochs for pretraining the first $50$ are solely based on reconstructions and the later $20$ epochs also account for log-probabilities. Following \citet{Brehmer2020mflows}, a loss term with weight $1e-3$ is added to constrain the latent representations. For \mult, an additional loss term with weight $10$ is added to encourage the mean responsibilities of the charts to be more uniform.

In Mocap1, the ground truth latent space is by definition $1$. As such, we perform an experiment where we calculate the pairwise distance between $5$ rotations of the original frame whose angles are uniformly distributed, solving one distance in each batch. While it is not clear whether the ground truth distances should be the same for these pairs, we should expect them to reflect the circular nature of the latent space.

\subsection{Compute Resources}

We mainly use a compute cluster to perform the experiments. The CPU is Intel Xeon Gold 6230 with 192GB memory, the GPU is Nvidia Volta V100 with 32GB memory.

Generally, for each CPU job we use $4$ CPU cores, while each GPU job uses one GPU and $10$ CPU cores. 
For jobs that involved training the flows, for Mocap datasets the jobs are run on GPU, otherwise they were run on CPU. The running times of the jobs may vary, but are generally within several hours. \singlem\ generally results in faster training than \single\ and \mult. Experiments involving evaluating the geometric and topological quantities were run on CPU, whose running times may vary and may need more than a day for those with ill-defined geometries. Some preliminary experiments were run, whose results did not make it to the paper, which naturally took up additional compute resources. The total amount of compute used is thus most likely thousands of CPU hours and hundreds of GPU hours.

\section{ADDITIONAL EXPERIMENTAL RESULTS}

In the result tables, we use Recons to denote the reconstruction errors, $\W$ to denote the Wasserstein distances, Exps to denote the qualities of exponential maps, Logs to denote the qualitities of logarithmic maps and Dists to denote the qualities of distances.

\subsection{Different Exponential Map Solvers for Multi-Chart Flows}
\label{sec:app-exp-ints}

Algorithm 1 in the main paper introduces one algorithm for solving for exponential maps based on multi-chart flows. We refer to that algorithm as \texttt{Euler}, and discuss two alternative algorithms, \texttt{Hard\_switch} and \texttt{Ambient}.

\paragraph{Hard\_switch.}

\begin{algorithm}[tb]
\caption{Hard\_switch algorithm for solving the exponential maps with a total of $T$ steps. $\text{Exp}_{c}$ denotes the exponential map induced by the $c$th flow.}
\While{$\arg\max\left(\text{resps}\left(\bx_{t}\right)\right) = c$ \text{and} $t<T$}{
    $\bx_{t+1} \gets \text{Exp}_{c}\left(\bx_{t},\bbv_{t}\right)$\;
    $t = t+1$\;
}
\label{alg:mult-exp-hard-switch}
\end{algorithm}

The naive algorithm for solving the exponential maps involves following along the geodesic in one chart at a time, and is presented in Algorithm~\ref{alg:mult-exp-hard-switch}. We refer to this algorithm as \textit{Hard\_switch}. In terms of numerical implementations, this can be achieved by specifying an event inside an ODE solver that terminates the current integration when the most responsible chart changes and jumps to another chart upon the event happens. With \textit{scipy.integrate.solve\_ivp}, we can specify the event in two ways: it can be a \textit{discrete} event checking whether the most responsible chart is the current one or a \textit{continuous} event checking the value of the difference between responsibilities of the current chart and the second most responsible one. 

\paragraph{Ambient.}

In the Euler algorithm, there is a hyperparameter $T$ that needs to be tuned. One interesting question to be asked is, what if we take the infinite limit? This question leads to an alternative algorithm where one averages over the velocities and accelerations, so as to directly perform the integrations of the geodesics in the ambient space. We derive the acceleration of geodesics when viewed in the ambient space.

\begin{theorem}
The acceleration of a geodesic can be expressed in the ambient space as
\begin{equation}
\ba_{\bX} = \frac{\partial\bJ}{\partial\bz}\bbv_{\bZ}\bbv_{\bZ} + \bJ \ba_{\bZ},
\end{equation}
where $\ba_{\bZ}$ is the acceleration as calculated in Theorem~\ref{thm:single-geodesic}.
\end{theorem}
\begin{proof}
Recall that the velocity in the latent space and the velocity in the ambient space are related by $\bbv_{\bX} = \bJ \bbv_{\bZ}$. Taking gradient with respect to $t$ on both sides, we have
\begin{equation}
\ba_{\bX} = \frac{\partial \left( \bJ \bbv_{\bZ} \right)}{\partial t} = \frac{\partial \bJ}{\partial t} \bbv_{\bZ} + \bJ \ba_{\bZ} = 
\frac{\partial \bJ}{\partial \bz} \bbv_{\bZ} \bbv_{\bZ} + \bJ \ba_{\bZ}.
\end{equation}
\end{proof}

\begin{algorithm}[tb]
\caption{Ambient algorithm for solving the exponential maps with a total of $T$ steps. $\text{Geo}_{c}$ denotes the function that takes in position and velocity and returns velocity and acceleration of the geodesic induced by the $c$th flow. $\text{Integral}_{\Delta t}$ denotes integrating the ODE for time $\Delta t$.}
\For{$t \gets 0$ \KwTo $T-1$}{
    $[\bbc , \br] \gets \text{filter}\left(\text{resps}\left(\bx_{t}\right), \text{resp\_threshold}\right)$ \Comment*[r]{\small \textnormal{// Get filtered responsibilities}}
    \For{$c$ in $\bbc$}{
        $\bbv_{t+1}^{c}, \ba_{t+1}^{c} \gets \text{Geo}_{c}\left(\bx_{t}, \bbv_{t}\right)$\;
    }
    $\bbv_{t+1} \gets \sum_{c=1}^{C}r_{c}\bbv_{t+1}^{c}$, $\ba_{t+1} \gets \sum_{c=1}^{C}r_{c}\ba_{t+1}^{c}$\;
    $\bx_{t+1},\bbv_{t+1} \gets \text{Integral}_{\Delta t}(\bx_{t},\bbv_{t},\bbv_{t+1},\ba_{t+1})$\;
}
\label{alg:mult-exp-ambient}
\end{algorithm}

Interestingly, observe that the first term also appears in Theorem~\ref{thm:single-geodesic}. Moreover, after $\ba_{\bZ}$ is calculated, we naturally have access to $\bJ$. As such, the acceleration in the ambient space can be calculated with little extra computational overhead. The resulting algorithm is outlined in Algorithm~\ref{alg:mult-exp-ambient}.

\begin{table*}
	\begin{center}
	\caption{The errors of the exponential map integrators in the form of mean $\pm$ std, lower is better. The best are highlighted in bold, and the ones with failed runs are italic. \textit{Hard\_switch} is notably worse than the other two.}
    \label{tbl:exp-integrators}
		\begin{tabular}{llll}
			\toprule
			Data & Euler & Ambient & Hard\_switch \\
			\midrule
			Sphere-U & $9.55 \!\cdot\! 10^{-3} \pm 6.31 \!\cdot\! 10^{-3}$ & $\mathbf{6.5 \!\cdot\! 10^{-3}} \pm 2.73 \!\cdot\! 10^{-3}$ & $1.39 \!\cdot\! 10^{-1} \pm 2.67 \!\cdot\! 10^{-2}$ \\
			Torus-U & $7.31 \!\cdot\! 10^{-2} \pm 1.26 \!\cdot\! 10^{-2}$ & $\mathit{3.5 \!\cdot\! 10^{-2}} \pm 3.37 \!\cdot\! 10^{-3}$ & $5.69 \!\cdot\! 10^{-1} \pm 1.05 \!\cdot\! 10^{-1}$ \\
			Sphere-M & $1.12 \!\cdot\! 10^{-2} \pm 6.33 \!\cdot\! 10^{-3}$ & $\mathbf{8.79 \!\cdot\! 10^{-3}} \pm 2.44 \!\cdot\! 10^{-4}$ & $2.24 \!\cdot\! 10^{-1} \pm 4.45 \!\cdot\! 10^{-2}$ \\
			Torus-M & $\mathbf{1.07 \!\cdot\! 10^{0}} \pm 2.39 \!\cdot\! 10^{-1}$ & $\mathit{2.14 \!\cdot\! 10^{16}} \pm 2.14 \!\cdot\! 10^{16}$ & $1.87 \!\cdot\! 10^{0} \pm 1.63 \!\cdot\! 10^{-1}$ \\
			\bottomrule
		\end{tabular}
	\end{center}
\end{table*}

\paragraph{Comparisons.}

We compare the \textit{discrete} variant of \textit{Hard\_switch}, \textit{Euler} and \textit{Ambient}. The results are shown in Table~\ref{tbl:exp-integrators}. In practice, we observe that Euler as outlined in Algorithm 1 offers reasonable performances and stabilities. Ambient yields good performances in some cases, but is unstable in some others. As expected, Hard\_switch consistently performs worse than Euler.

\subsection{Results on Circle}

\begin{table*}
	\begin{center}
	\caption{Evaluation metrics on circle in the form of mean $\pm$ std, lower is better. The best are highlighted in bold.}
    \label{tbl:circle-results}
		\begin{tabular}{lll}
			\toprule
			Method & Recons & $\W$ \\
			\midrule
			\singlem & $3.65 \!\cdot\! 10^{-3} \pm 3.66 \!\cdot\! 10^{-3}$ & $6.87 \!\cdot\! 10^{-1} \pm 4.3 \!\cdot\! 10^{-1}$ \\
			\single & $4.68 \!\cdot\! 10^{-3} \pm 5.63 \!\cdot\! 10^{-3}$ & $1.1 \!\cdot\! 10^{-1} \pm 2.49 \!\cdot\! 10^{-2}$ \\
			\multdirect & $7.73 \!\cdot\! 10^{-7} \pm 4.32 \!\cdot\! 10^{-7}$ & $\mathbf{8.24 \!\cdot\! 10^{-2}} \pm 2.88 \!\cdot\! 10^{-2}$ \\
			\mult & $\mathbf{5.48 \!\cdot\! 10^{-7}} \pm 1.2 \!\cdot\! 10^{-7}$ & $8.56 \!\cdot\! 10^{-2} \pm 2.79 \!\cdot\! 10^{-2}$ \\
			\bottomrule
		\end{tabular}
	\end{center}
\end{table*}

\begin{figure}
    \centering
    \includegraphics[width=0.7\textwidth]{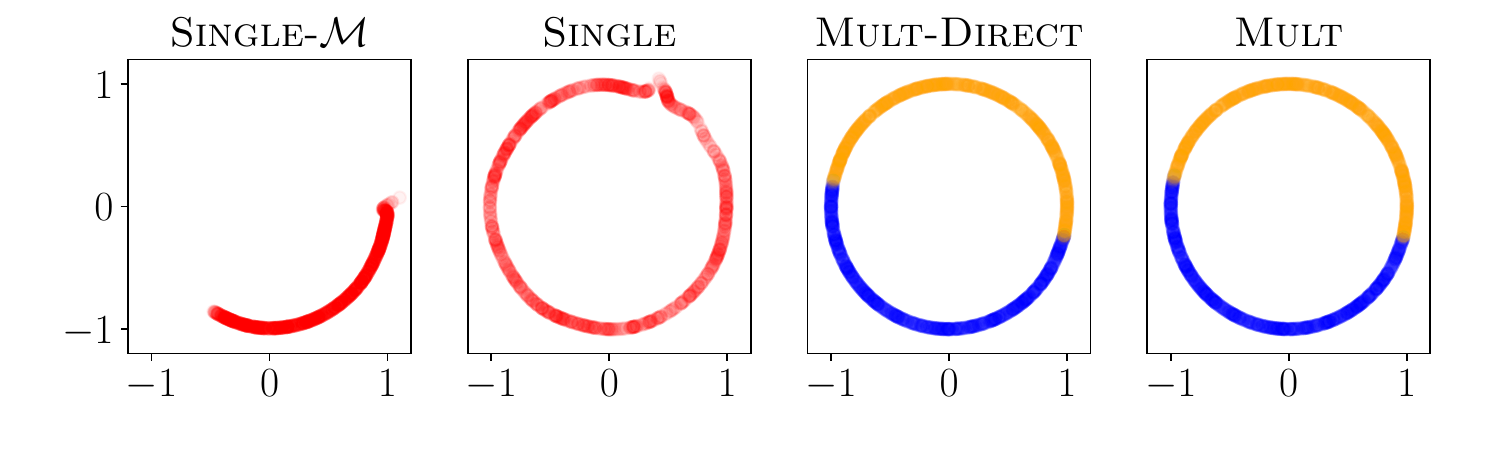}
    \caption{From left to right: \singlem, \single, \multdirect\ and \mult. While single-chart flows fail to cover the circle, multi-chart flows are able to provide samples across the circle.}
    \label{fig:circle_samples}
\end{figure}

\begin{figure}
    \centering
    \includegraphics[width=0.7\textwidth]{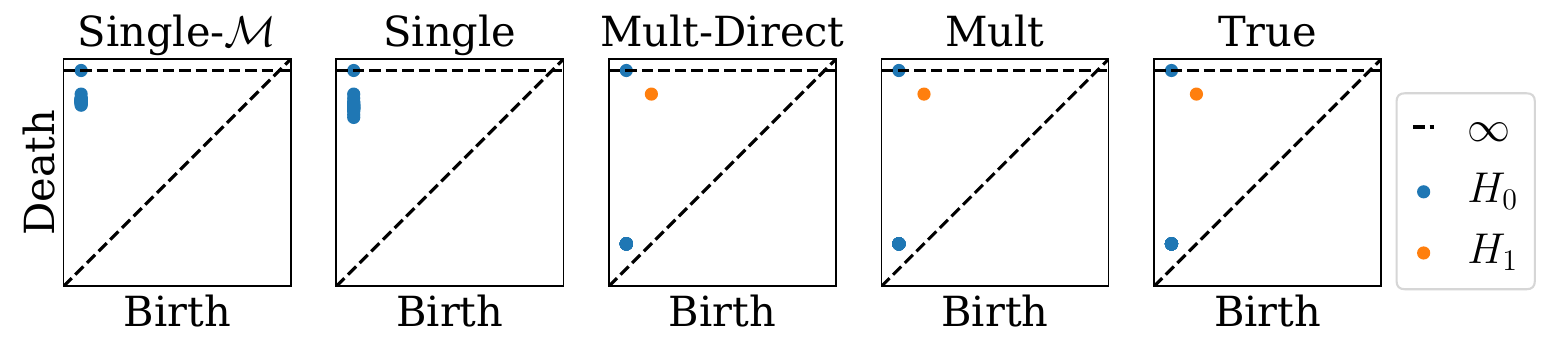}
    \caption{From left to right: example persistence diagrams on \singlem, \single, \multdirect\ and \mult. Using a single chart results in pathological failures, where the model fails to cover the entire circle. Only when using multiple charts the models learn the correct topology.}
    \label{fig:circle_dgms}
\end{figure}

\begin{figure}
    \centering
    \includegraphics[width=0.7\linewidth]{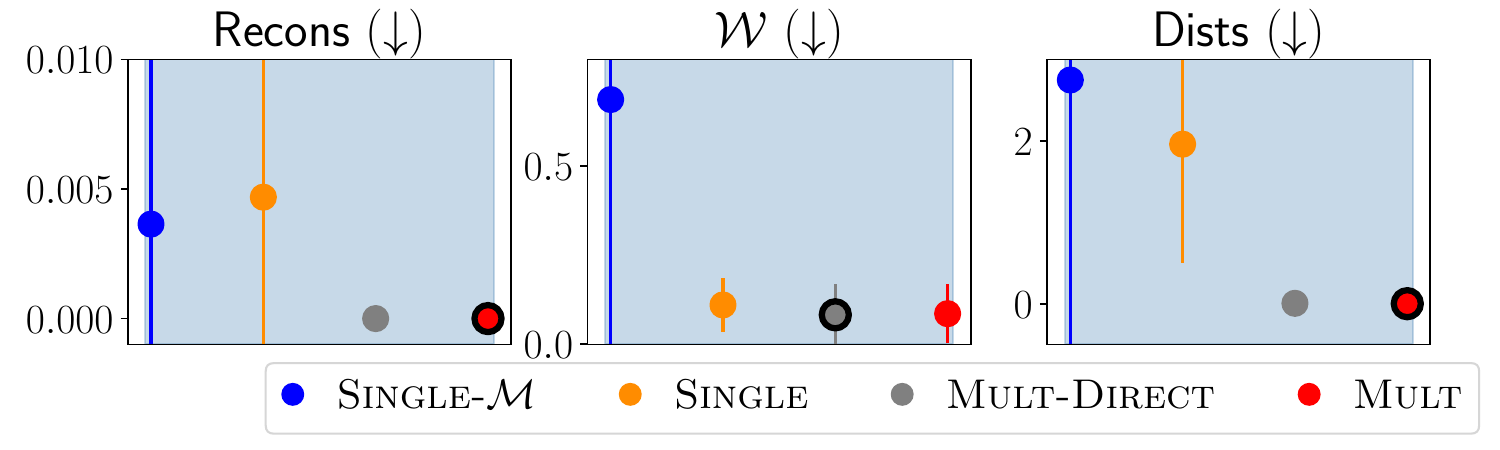}
    \caption{Evaluation metrics of different methods on Circle. Generally, \multdirect\ and \mult\ yield roughly the same level of performances while clearly outperforming \singlem\ and \single.}
    \label{fig:circle_results}
\end{figure}

The evaluation metrics of the different algorithms trained to model the uniform distribution on a circle are shown in Figure~\ref{fig:circle_results}, and the numerical results can be found in Table~\ref{tbl:circle-results}. We do not observe a qualitative difference between direct MLE and EM. We additionally report example samples generated by the models in Figure~\ref{fig:circle_samples} and example persistence diagrams computed based on the models in Figure~\ref{fig:circle_dgms}. Both \multdirect\ and \mult\ are able to capture the topology well.

\subsection{Results on Sphere and Torus}
\label{sec:app-res-sphere-torus}

We report the numerical results of the different models on sphere and torus in Table~\ref{tbl:sphere-torus-single_m-0}, Table~\ref{tbl:sphere-torus-single-0}, Table~\ref{tbl:sphere-torus-mult-0}, Table~\ref{tbl:sphere-torus-single_m-1}, Table~\ref{tbl:sphere-torus-single-1} and Table~\ref{tbl:sphere-torus-mult-1}. It is entirely possible for a geodesic solver to return a line that is quite different from the \textit{true} geodesic while still being reasonable, resulting in rather different logarithmic maps and longer distances; Figure~\ref{fig:torus-log} shows such an example. We additionally provide example persistence diagrams for Sphere-M, Torus-U and Torus-M in Figure~\ref{fig:supp_persistence_diagrams}; note that example diagrams for Sphere-U were already provided in the main paper. It is clear that \mult\ consistently provides the most faithful representations of the underlying topology.

\begin{figure*}
    \centering
    \includegraphics[width=0.3\textwidth]{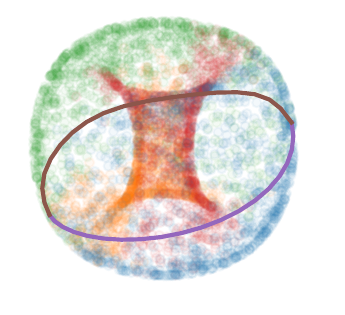}
    \caption{An example of geodesic solver failure on torus. The solver finds the path on the opposite side, resulting in a drastically different logarithmic map and a longer distance.}
    \label{fig:torus-log}
\end{figure*}

\begin{table*}
	\begin{center}
	\caption{Evaluation metrics concerning modeling of \singlem\ on sphere and torus in the form of mean $\pm$ std, lower is better.}
    \label{tbl:sphere-torus-single_m-0}
		\begin{tabular}{lll}
			\toprule
			Manifold & Recons & $\W$ \\
			\midrule
			Sphere-U & $2.49 \!\cdot\! 10^{-4} \pm 2.41 \!\cdot\! 10^{-4}$ & $1.31 \!\cdot\! 10^{-1} \pm 1.28 \!\cdot\! 10^{-2}$ \\
			Torus-U & $3.0 \!\cdot\! 10^{-2} \pm 2.79 \!\cdot\! 10^{-2}$ & $2.96 \!\cdot\! 10^{-1} \pm 4.16 \!\cdot\! 10^{-2}$ \\
			Sphere-M & $1.69 \!\cdot\! 10^{-3} \pm 4.51 \!\cdot\! 10^{-3}$ & $1.7 \!\cdot\! 10^{-1} \pm 7.34 \!\cdot\! 10^{-2}$ \\
			Torus-M & $5.39 \!\cdot\! 10^{-2} \pm 4.59 \!\cdot\! 10^{-2}$ & $9.93 \!\cdot\! 10^{-1} \pm 1.52 \!\cdot\! 10^{0}$ \\
			\bottomrule
		\end{tabular}
	\end{center}
\end{table*}

\begin{table*}
	\begin{center}
	\caption{Evaluation metrics concerning modeling of \single\ on sphere and torus in the form of mean $\pm$ std, lower is better.}
    \label{tbl:sphere-torus-single-0}
		\begin{tabular}{lll}
			\toprule
			Manifold & Recons & $\W$ \\
			\midrule
			Sphere-U & $8.2 \!\cdot\! 10^{-5} \pm 5.42 \!\cdot\! 10^{-5}$ & $1.31 \!\cdot\! 10^{0} \pm 5.7 \!\cdot\! 10^{0}$ \\
			Torus-U & $2.7 \!\cdot\! 10^{-2} \pm 2.73 \!\cdot\! 10^{-2}$ & $4.38 \!\cdot\! 10^{-1} \pm 1.51 \!\cdot\! 10^{-1}$ \\
			Sphere-M & $2.09 \!\cdot\! 10^{-4} \pm 2.06 \!\cdot\! 10^{-4}$ & $5.21 \!\cdot\! 10^{0} \pm 1.86 \!\cdot\! 10^{1}$ \\
			Torus-M & $5.28 \!\cdot\! 10^{-2} \pm 3.36 \!\cdot\! 10^{-2}$ & $1.1 \!\cdot\! 10^{0} \pm 9.57 \!\cdot\! 10^{-1}$ \\
			\bottomrule
		\end{tabular}
	\end{center}
\end{table*}

\begin{table*}
	\begin{center}
	\caption{Evaluation metrics concerning modeling of \mult\ on sphere and torus in the form of mean $\pm$ std, lower is better.}
    \label{tbl:sphere-torus-mult-0}
		\begin{tabular}{lll}
			\toprule
			Manifold & Recons & $\W$ \\
			\midrule
			Sphere-U & $1.87 \!\cdot\! 10^{-6} \pm 5.91 \!\cdot\! 10^{-7}$ & $1.28 \!\cdot\! 10^{-1} \pm 1.13 \!\cdot\! 10^{-2}$ \\
			Torus-U & $2.43 \!\cdot\! 10^{-5} \pm 7.68 \!\cdot\! 10^{-6}$ & $2.34 \!\cdot\! 10^{-1} \pm 2.39 \!\cdot\! 10^{-2}$ \\
			Sphere-M & $1.42 \!\cdot\! 10^{-6} \pm 2.44 \!\cdot\! 10^{-7}$ & $1.3 \!\cdot\! 10^{-1} \pm 1.49 \!\cdot\! 10^{-2}$ \\
			Torus-M & $6.7 \!\cdot\! 10^{-5} \pm 3.83 \!\cdot\! 10^{-5}$ & $3.36 \!\cdot\! 10^{-1} \pm 5.6 \!\cdot\! 10^{-2}$ \\
			\bottomrule
		\end{tabular}
	\end{center}
\end{table*}

\begin{table*}
	\begin{center}
	\caption{Evaluation metrics concerning geometry of \singlem\ on sphere and torus in the form of mean $\pm$ std, lower is better. The ones with failed runs are italic.}
    \label{tbl:sphere-torus-single_m-1}
		\begin{tabular}{llll}
			\toprule
			Manifold & Exps & Logs & Dists \\
			\midrule
			Sphere-U & $\mathit{5.32 \!\cdot\! 10^{12}} \pm 1.41 \!\cdot\! 10^{13}$ & $1.56 \!\cdot\! 10^{0} \pm 1.54 \!\cdot\! 10^{0}$ & $1.42 \!\cdot\! 10^{-1} \pm 2.86 \!\cdot\! 10^{-1}$ \\
			Torus-U & $\mathit{4.28 \!\cdot\! 10^{18}} \pm 1.13 \!\cdot\! 10^{19}$ & $3.04 \!\cdot\! 10^{1} \pm 8.44 \!\cdot\! 10^{0}$ & $9.74 \!\cdot\! 10^{0} \pm 5.2 \!\cdot\! 10^{0}$ \\
			Sphere-M & $7.55 \!\cdot\! 10^{9} \pm 2.26 \!\cdot\! 10^{10}$ & $4.58 \!\cdot\! 10^{3} \pm 1.37 \!\cdot\! 10^{4}$ & $4.57 \!\cdot\! 10^{3} \pm 1.37 \!\cdot\! 10^{4}$ \\
			Torus-M & $\mathit{4.93 \!\cdot\! 10^{4}} \pm 1.18 \!\cdot\! 10^{5}$ & $3.92 \!\cdot\! 10^{2} \pm 6.8 \!\cdot\! 10^{2}$ & $3.63 \!\cdot\! 10^{2} \pm 6.73 \!\cdot\! 10^{2}$ \\
			\bottomrule
		\end{tabular}
	\end{center}
\end{table*}

\begin{table*}
	\begin{center}
	\caption{Evaluation metrics concerning geometry of \single\ on sphere and torus in the form of mean $\pm$ std, lower is better. The ones with failed runs are italic.}
    \label{tbl:sphere-torus-single-1}
		\begin{tabular}{llll}
			\toprule
			Manifold & Exps & Logs & Dists \\
			\midrule
			Sphere-U & $\mathit{4.25 \!\cdot\! 10^{23}} \pm 1.2 \!\cdot\! 10^{24}$ & $1.32 \!\cdot\! 10^{0} \pm 8.96 \!\cdot\! 10^{-1}$ & $8.62 \!\cdot\! 10^{-2} \pm 7.73 \!\cdot\! 10^{-2}$ \\
			Torus-U & $5.54 \!\cdot\! 10^{17} \pm 1.66 \!\cdot\! 10^{18}$ & $3.06 \!\cdot\! 10^{1} \pm 1.46 \!\cdot\! 10^{1}$ & $1.05 \!\cdot\! 10^{1} \pm 9.45 \!\cdot\! 10^{0}$ \\
			Sphere-M & $5.46 \!\cdot\! 10^{3} \pm 1.61 \!\cdot\! 10^{4}$ & $1.63 \!\cdot\! 10^{0} \pm 1.18 \!\cdot\! 10^{0}$ & $1.21 \!\cdot\! 10^{-1} \pm 1.38 \!\cdot\! 10^{-1}$ \\
			Torus-M & $5.93 \!\cdot\! 10^{10} \pm 1.78 \!\cdot\! 10^{11}$ & $3.42 \!\cdot\! 10^{3} \pm 5.78 \!\cdot\! 10^{3}$ & $3.38 \!\cdot\! 10^{3} \pm 5.76 \!\cdot\! 10^{3}$ \\
			\bottomrule
		\end{tabular}
	\end{center}
\end{table*}

\begin{table*}
	\begin{center}
	\caption{Evaluation metrics concerning geometry of \mult\ on sphere and torus in the form of mean $\pm$ std, lower is better.}
    \label{tbl:sphere-torus-mult-1}
		\begin{tabular}{llll}
			\toprule
			Manifold & Exps & Logs & Dists \\
			\midrule
			Sphere-U & $1.39 \!\cdot\! 10^{-2} \pm 1.51 \!\cdot\! 10^{-2}$ & $1.11 \!\cdot\! 10^{-1} \pm 3.28 \!\cdot\! 10^{-2}$ & $3.34 \!\cdot\! 10^{-4} \pm 2.21 \!\cdot\! 10^{-4}$ \\
			Torus-U & $1.87 \!\cdot\! 10^{-1} \pm 1.75 \!\cdot\! 10^{-1}$ & $4.55 \!\cdot\! 10^{0} \pm 1.85 \!\cdot\! 10^{0}$ & $2.39 \!\cdot\! 10^{-1} \pm 1.82 \!\cdot\! 10^{-1}$ \\
			Sphere-M & $5.13 \!\cdot\! 10^{-3} \pm 4.51 \!\cdot\! 10^{-3}$ & $1.4 \!\cdot\! 10^{-1} \pm 4.32 \!\cdot\! 10^{-2}$ & $6.03 \!\cdot\! 10^{-4} \pm 3.16 \!\cdot\! 10^{-4}$ \\
			Torus-M & $1.94 \!\cdot\! 10^{0} \pm 2.13 \!\cdot\! 10^{0}$ & $3.19 \!\cdot\! 10^{0} \pm 1.05 \!\cdot\! 10^{0}$ & $7.1 \!\cdot\! 10^{-2} \pm 2.87 \!\cdot\! 10^{-2}$ \\
			\bottomrule
		\end{tabular}
	\end{center}
\end{table*}

\begin{figure}
    \centering
    \includegraphics[width=0.7\linewidth]{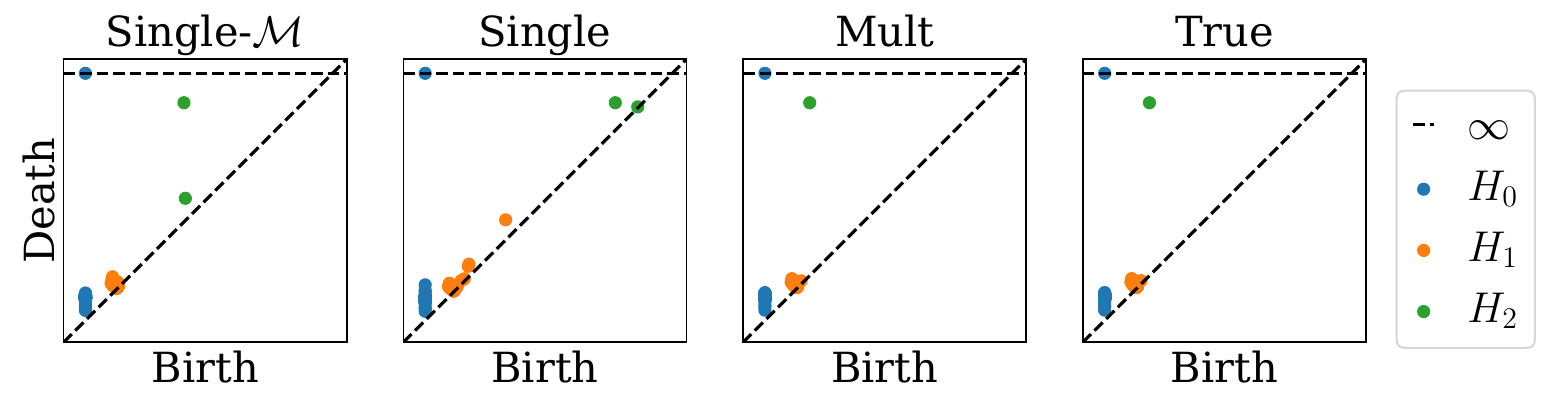} \\
    \includegraphics[width=0.7\linewidth]{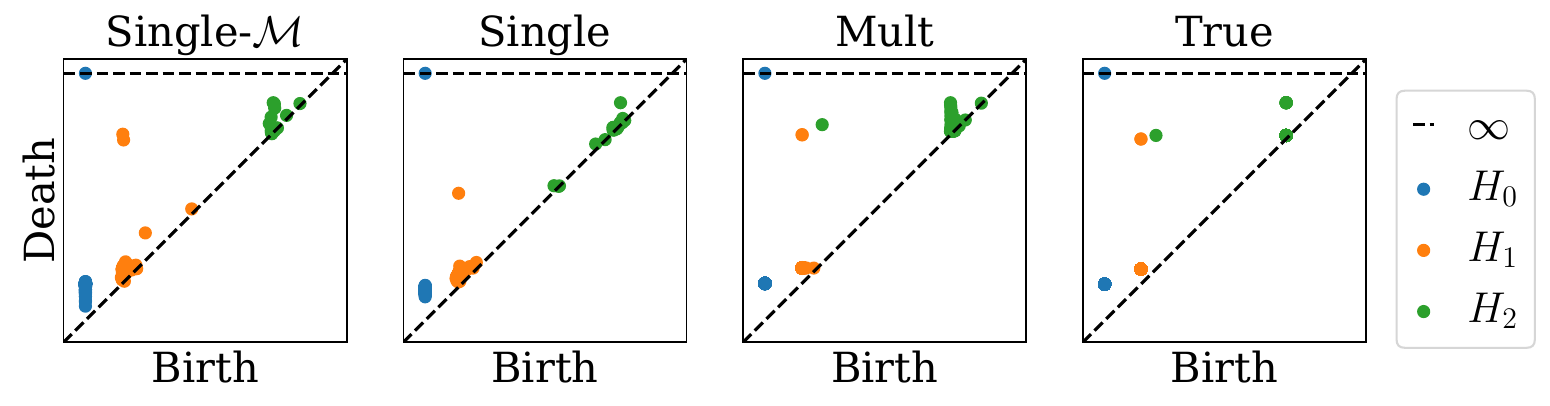} \\
    \includegraphics[width=0.7\linewidth]{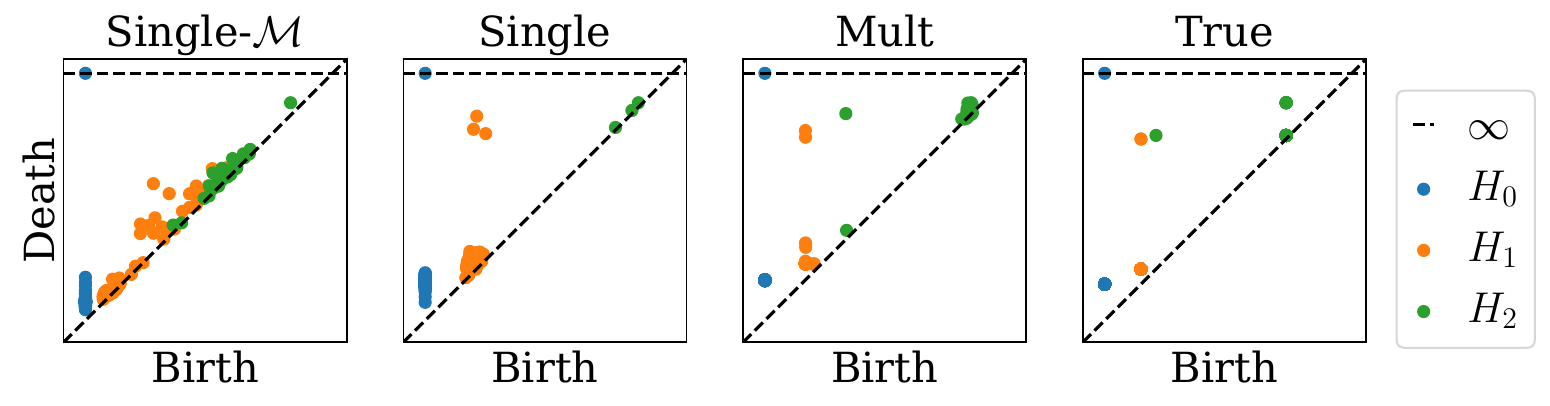}
    \caption{From top to bottom: example persistence diagrams induced by different models on Sphere-M, Torus-U and Torus-M.}
    \label{fig:supp_persistence_diagrams}
\end{figure}

\subsection{Results on Triangular Meshes}

We report the numerical results of \singlem, \single\ and \mult\ on triangular meshes in Table~\ref{tbl:mesh-single_m}, Table~\ref{tbl:mesh-single} and Table~\ref{tbl:mesh-mult}. We additionally report example persistence diagrams obtained under these settings in Figure~\ref{fig:supp_mesh_persistence_diagrams}. Due to the uniform nature of the evaluation points and their relatively small number, the diagrams may not that clearly reflect the underlying topology.

\begin{table*}
	\begin{center}
	\caption{Evaluation metrics of \singlem\ on triangular meshes in the form of mean $\pm$ std, lower is better.}
    \label{tbl:mesh-single_m}
		\begin{tabular}{llll}
			\toprule
			Manifold & Recons & $\W$ & Dists \\
			\midrule
			7-8 cube & $2.61 \!\cdot\! 10^{-4} \pm 2.53 \!\cdot\! 10^{-4}$ & $1.97 \!\cdot\! 10^{-1} \pm 1.81 \!\cdot\! 10^{-2}$ & $2.21 \!\cdot\! 10^{-1} \pm 1.57 \!\cdot\! 10^{-1}$ \\
			Busted & $4.97 \!\cdot\! 10^{-4} \pm 1.99 \!\cdot\! 10^{-4}$ & $1.97 \!\cdot\! 10^{-1} \pm 1.7 \!\cdot\! 10^{-2}$ & $5.95 \!\cdot\! 10^{-1} \pm 5.17 \!\cdot\! 10^{-1}$ \\
			Rock & $3.56 \!\cdot\! 10^{-4} \pm 1.49 \!\cdot\! 10^{-4}$ & $1.71 \!\cdot\! 10^{-1} \pm 1.57 \!\cdot\! 10^{-2}$ & $2.43 \!\cdot\! 10^{-1} \pm 1.23 \!\cdot\! 10^{-1}$ \\
			\bottomrule
		\end{tabular}
	\end{center}
\end{table*}

\begin{table*}
	\begin{center}
	\caption{Evaluation metrics of \single\ on triangular meshes in the form of mean $\pm$ std, lower is better.}
    \label{tbl:mesh-single}
		\begin{tabular}{llll}
			\toprule
			Manifold & Recons & $\W$ & Dists \\
			\midrule
			7-8 cube & $5.69 \!\cdot\! 10^{-3} \pm 1.16 \!\cdot\! 10^{-2}$ & $2.45 \!\cdot\! 10^{-1} \pm 6.65 \!\cdot\! 10^{-2}$ & $1.91 \!\cdot\! 10^{3} \pm 5.47 \!\cdot\! 10^{3}$ \\
			Busted & $4.2 \!\cdot\! 10^{-4} \pm 2.54 \!\cdot\! 10^{-4}$ & $2.06 \!\cdot\! 10^{-1} \pm 2.38 \!\cdot\! 10^{-2}$ & $3.15 \!\cdot\! 10^{-1} \pm 2.13 \!\cdot\! 10^{-1}$ \\
			Rock & $2.56 \!\cdot\! 10^{-4} \pm 1.31 \!\cdot\! 10^{-4}$ & $2.02 \!\cdot\! 10^{-1} \pm 5.65 \!\cdot\! 10^{-2}$ & $1.57 \!\cdot\! 10^{-1} \pm 1.5 \!\cdot\! 10^{-1}$ \\
			\bottomrule
		\end{tabular}
	\end{center}
\end{table*}

\begin{table*}
	\begin{center}
	\caption{Evaluation metrics of \mult\ on triangular meshes in the form of mean $\pm$ std, lower is better.}
    \label{tbl:mesh-mult}
		\begin{tabular}{llll}
			\toprule
			Manifold & Recons & $\W$ & Dists \\
			\midrule
			7-8 cube & $1.64 \!\cdot\! 10^{-5} \pm 3.37 \!\cdot\! 10^{-6}$ & $1.97 \!\cdot\! 10^{-1} \pm 2.08 \!\cdot\! 10^{-2}$ & $9.11 \!\cdot\! 10^{-2} \pm 7.41 \!\cdot\! 10^{-3}$ \\
			Busted & $5.63 \!\cdot\! 10^{-5} \pm 1.63 \!\cdot\! 10^{-5}$ & $2.1 \!\cdot\! 10^{-1} \pm 2.73 \!\cdot\! 10^{-2}$ & $4.76 \!\cdot\! 10^{-2} \pm 6.72 \!\cdot\! 10^{-3}$ \\
			Rock & $1.36 \!\cdot\! 10^{-5} \pm 2.94 \!\cdot\! 10^{-6}$ & $1.78 \!\cdot\! 10^{-1} \pm 2.7 \!\cdot\! 10^{-2}$ & $6.03 \!\cdot\! 10^{-2} \pm 1.99 \!\cdot\! 10^{-2}$ \\
			\bottomrule
		\end{tabular}
	\end{center}
\end{table*}

\begin{figure}
    \centering
    \includegraphics[width=0.7\linewidth]{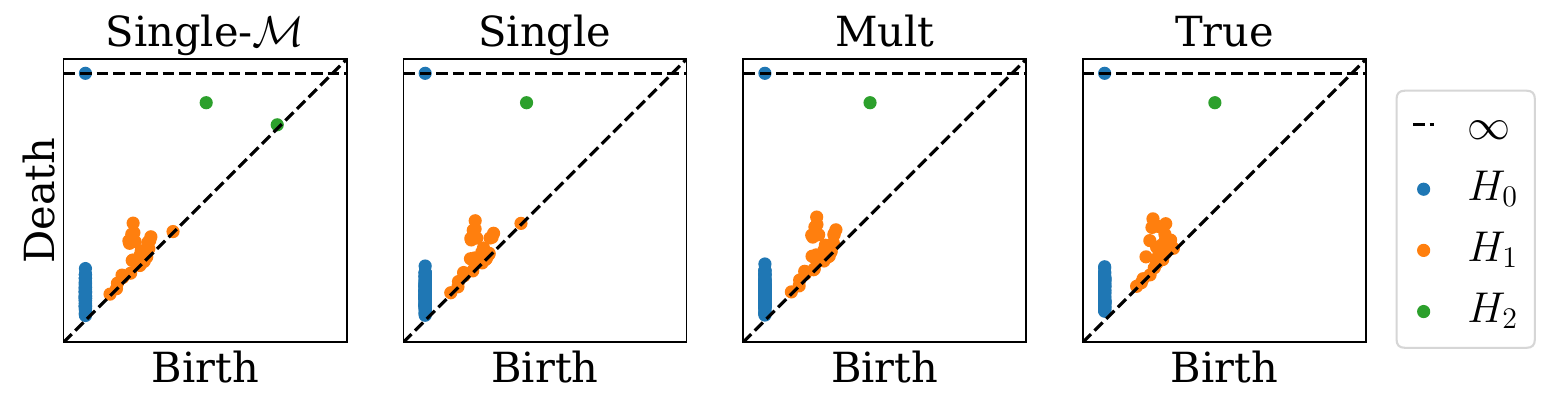} \\
    \includegraphics[width=0.7\linewidth]{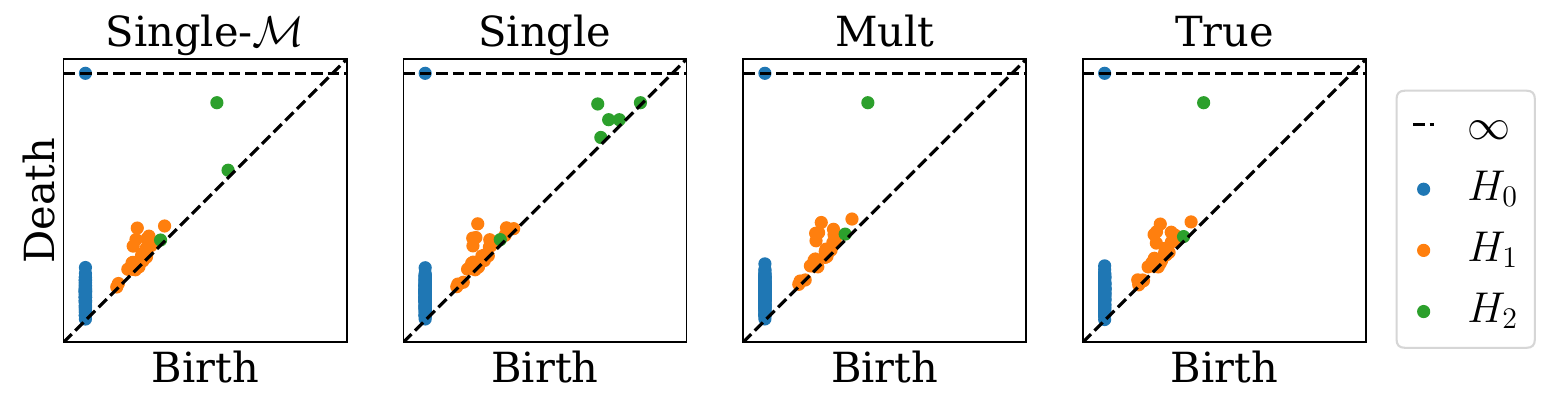} \\
    \includegraphics[width=0.7\linewidth]{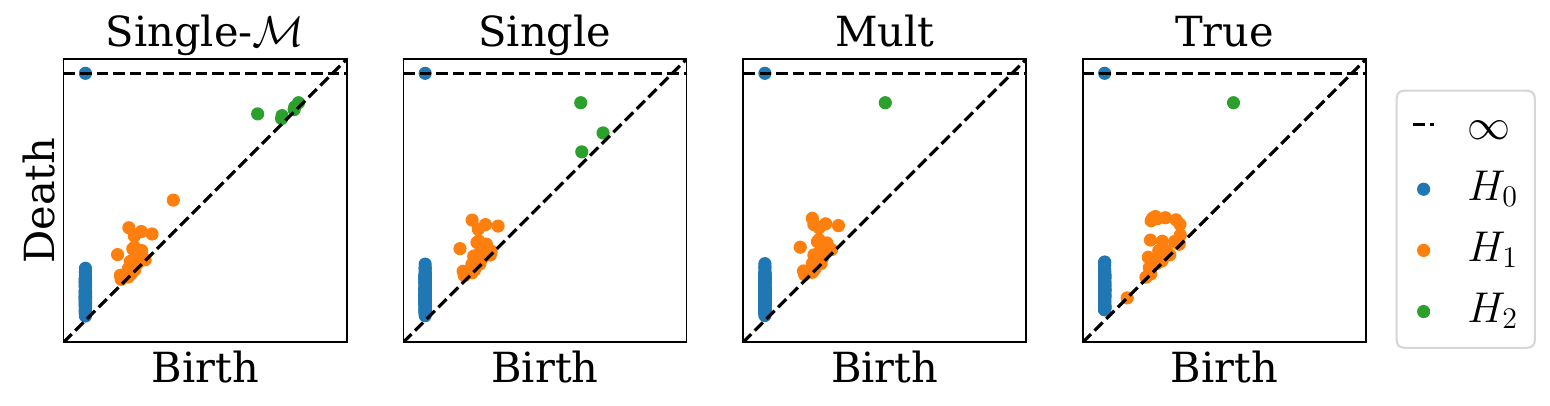}
    \caption{From top to bottom: example persistence diagrams induced by different models on 7-8th cube, Busted and Rock.}
    \label{fig:supp_mesh_persistence_diagrams}
\end{figure}

\subsection{Results on Mocap Data}

The numerical results of \singlem, \single\ and \mult\ on Mocap data are reported in Table~\ref{tbl:mocap-single_m}, Table~\ref{tbl:mocap-single} and Table~\ref{tbl:mocap-mult}, respectively.

\begin{table*}
	\begin{center}
	\caption{Evaluation metrics of \singlem\ on Mocap in the form of mean $\pm$ std, lower is better.}
    \label{tbl:mocap-single_m}
		\begin{tabular}{lll}
			\toprule
			Manifold & Recons & $\W$ \\
			\midrule
			Mocap1 & $5.5 \!\cdot\! 10^{-2} \pm 7.15 \!\cdot\! 10^{-2}$ & $8.05 \!\cdot\! 10^{-1} \pm 2.25 \!\cdot\! 10^{-1}$ \\
			Mocap3 & $3.76 \!\cdot\! 10^{0} \pm 2.13 \!\cdot\! 10^{0}$ & $2.87 \!\cdot\! 10^{0} \pm 2.39 \!\cdot\! 10^{-1}$ \\
			\bottomrule
		\end{tabular}
	\end{center}
\end{table*}

\begin{table*}
	\begin{center}
	\caption{Evaluation metrics of \single\ on Mocap in the form of mean $\pm$ std, lower is better.}
    \label{tbl:mocap-single}
		\begin{tabular}{lll}
			\toprule
			Manifold & Recons & $\W$ \\
			\midrule
			Mocap1 & $1.14 \!\cdot\! 10^{-2} \pm 1.28 \!\cdot\! 10^{-2}$ & $8.39 \!\cdot\! 10^{-1} \pm 2.44 \!\cdot\! 10^{-1}$ \\
			Mocap3 & $3.5 \!\cdot\! 10^{0} \pm 2.14 \!\cdot\! 10^{0}$ & $3.09 \!\cdot\! 10^{0} \pm 3.28 \!\cdot\! 10^{-1}$ \\
			\bottomrule
		\end{tabular}
	\end{center}
\end{table*}

\begin{table*}
	\begin{center}
	\caption{Evaluation metrics of \mult\ on Mocap in the form of mean $\pm$ std, lower is better.}
    \label{tbl:mocap-mult}
		\begin{tabular}{lll}
			\toprule
			Manifold & Recons & $\W$ \\
			\midrule
			Mocap1 & $6.79 \!\cdot\! 10^{-4} \pm 9.64 \!\cdot\! 10^{-5}$ & $7.3 \!\cdot\! 10^{-1} \pm 1.98 \!\cdot\! 10^{-1}$ \\
			Mocap3 & $3.97 \!\cdot\! 10^{-3} \pm 2.47 \!\cdot\! 10^{-4}$ & $2.47 \!\cdot\! 10^{0} \pm 7.08 \!\cdot\! 10^{-2}$ \\
			\bottomrule
		\end{tabular}
	\end{center}
\end{table*}

\subsection{Results on Stiefel Manifold}

We report results on an additional Riemannian manifold, the Stiefel manifold, which is the manifold formed by the set of all orthonormal $p$ frames in an $n$ dimensional space, based on Geomstats \citep{miolane2020geomstats_jmlr,miolane2024geomstats_software}. Note that the Lusternik-Schnirelmann category of Stiefel manifolds are generally known \citep{nishimoto2007ls_stiefel}, such that we have theoretical guarantees that using a smaller number of charts can cover the manifolds. For \mult, we always use $4$ charts. In the following we specify the settings, listing the intrinsic dimensionality $d$, the ambient dimensionality $D$, the number of layers in $\bg$ num\_$\bg$ and the number of layers in $\bh$ num\_$\bh$ in each chart:
\begin{enumerate}
\item $n=3, p=3$: $d=3$, $D=9$, \singlem\ and \single: num\_$\bg = 12$, num\_$\bh = 36$, \mult: num\_$\bg = 3$, num\_$\bh = 9$, 
\item $n=6, p=2$: $d=9$, $D=12$, \singlem\ and \single: num\_$\bg = 24$, num\_$\bh = 48$, \mult: num\_$\bg = 6$, num\_$\bh = 12$,
\item $n=6, p=3$: $d=12$, $D=18$, \singlem\ and \single: num\_$\bg = 24$, num\_$\bh = 48$, \mult: num\_$\bg = 6$, num\_$\bh = 12$.
\end{enumerate}

For the metric induced by ambient Euclidean space, the exponential maps are tractable while the logarithmic maps are not. We thus perform evaluation on the exponential maps, where we sample $320$ pairs of initial positions and initial velocities and solve $32$ of them in a batch. The initial positions are directly sampled from the data distribution, while the initial velocities are first sampled from the data distribution before projecting onto the tangent spaces corresponding to the initial positions. We observe that \single\ and \singlem\ can lead to challenging integration problems, and we thus set a running time limit of $36$ hours. Jobs that exceed the limit are considered failed and the overall result is considered bad.

We report the results in Figure~\ref{fig:stiefel_results}. The experimental results of \singlem, \single\ and \mult\ are shown in Table~\ref{tbl:stiefel-single_m}, Table~\ref{tbl:stiefel-single} and Table~\ref{tbl:stiefel-mult}, respectively. We observe that \mult\ is always the best in terms of reconstructions and yields good exponential map solutions, while being highly competitive and among the best for sample quality.

\begin{figure}
    \centering
    \includegraphics[width=0.7\linewidth]{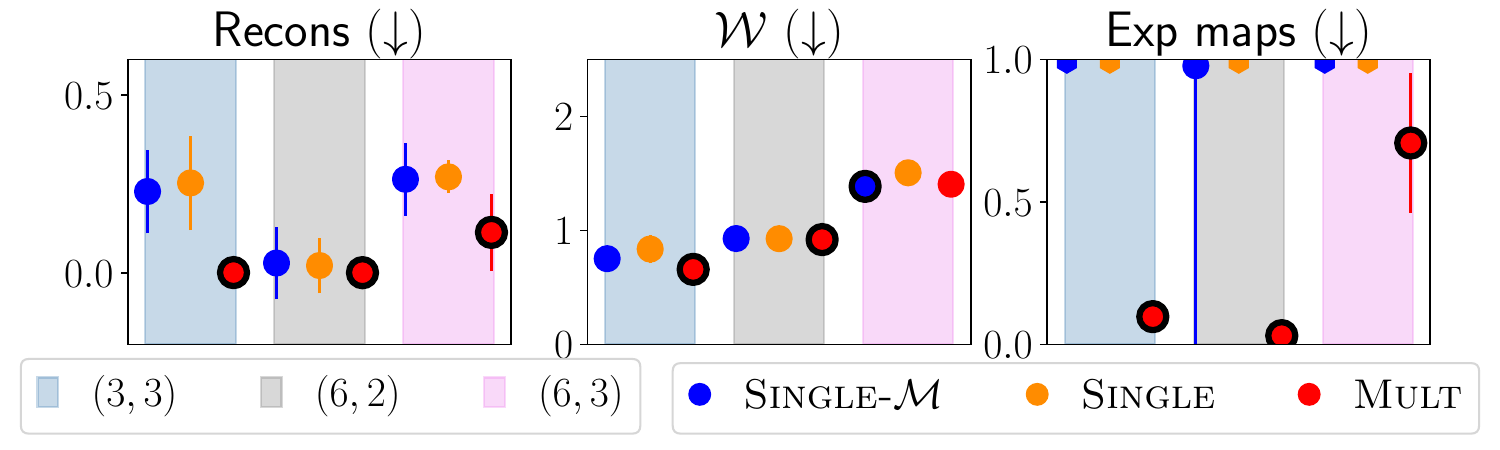}
    \caption{Evaluation metrics of different methods on Stiefel. \mult\ always yields the best reconstructions and overall the best exponential map solutions, while remaining highly competitive in terms of sample quality.}
    \label{fig:stiefel_results}
\end{figure}

\begin{table*}
	\begin{center}
	\caption{Evaluation metrics of \singlem\ on Stiefel in the form of mean $\pm$ std, lower is better. The ones with failed runs are italic.}
    \label{tbl:stiefel-single_m}
		\begin{tabular}{llll}
			\toprule
			$n, p$ & Recons & $\mathcal{W}$ & Exps \\
			\midrule
			n=3, p=3 & $2.29 \!\cdot\! 10^{-1} \pm 3.92 \!\cdot\! 10^{-2}$ & $7.53 \!\cdot\! 10^{-1} \pm 1.74 \!\cdot\! 10^{-2}$ & $\mathit{4.36 \!\cdot\! 10^{8}} \pm 3.32 \!\cdot\! 10^{8}$ \\
			n=6, p=2 & $2.87 \!\cdot\! 10^{-2} \pm 3.39 \!\cdot\! 10^{-2}$ & $9.29 \!\cdot\! 10^{-1} \pm 7.43 \!\cdot\! 10^{-3}$ & $9.78 \!\cdot\! 10^{-1} \pm 1.21 \!\cdot\! 10^{0}$ \\
			n=6, p=3 & $2.64 \!\cdot\! 10^{-1} \pm 3.39 \!\cdot\! 10^{-2}$ & $1.39 \!\cdot\! 10^{0} \pm 3.21 \!\cdot\! 10^{-3}$ & $\mathit{nan} \pm nan$ \\
			\bottomrule
		\end{tabular}
	\end{center}
\end{table*}

\begin{table*}
	\begin{center}
	\caption{Evaluation metrics of \single\ on Stiefel in the form of mean $\pm$ std, lower is better. The ones with failed runs are italic.}
    \label{tbl:stiefel-single}
		\begin{tabular}{llll}
			\toprule
			$n, p$ & Recons & $\mathcal{W}$ & Exps \\
			\midrule
			n=3, p=3 & $2.53 \!\cdot\! 10^{-1} \pm 4.41 \!\cdot\! 10^{-2}$ & $8.37 \!\cdot\! 10^{-1} \pm 4.17 \!\cdot\! 10^{-2}$ & $5.56 \!\cdot\! 10^{6} \pm 7.86 \!\cdot\! 10^{6}$ \\
			n=6, p=2 & $2.16 \!\cdot\! 10^{-2} \pm 2.59 \!\cdot\! 10^{-2}$ & $9.28 \!\cdot\! 10^{-1} \pm 2.14 \!\cdot\! 10^{-2}$ & $\mathit{2.29 \!\cdot\! 10^{-2}} \pm 3.19 \!\cdot\! 10^{-3}$ \\
			n=6, p=3 & $2.7 \!\cdot\! 10^{-1} \pm 1.53 \!\cdot\! 10^{-2}$ & $1.51 \!\cdot\! 10^{0} \pm 3.66 \!\cdot\! 10^{-2}$ & $\mathit{nan} \pm nan$ \\
			\bottomrule
		\end{tabular}
	\end{center}
\end{table*}

\begin{table*}
	\begin{center}
	\caption{Evaluation metrics of \mult\ on Stiefel in the form of mean $\pm$ std, lower is better.}
    \label{tbl:stiefel-mult}
		\begin{tabular}{llll}
			\toprule
			$n, p$ & Recons & $\mathcal{W}$ & Exps \\
			\midrule
			n=3, p=3 & $1.86 \!\cdot\! 10^{-3} \pm 4.72 \!\cdot\! 10^{-4}$ & $6.59 \!\cdot\! 10^{-1} \pm 3.05 \!\cdot\! 10^{-2}$ & $9.78 \!\cdot\! 10^{-2} \pm 3.33 \!\cdot\! 10^{-3}$ \\
			n=6, p=2 & $1.58 \!\cdot\! 10^{-3} \pm 3.38 \!\cdot\! 10^{-4}$ & $9.2 \!\cdot\! 10^{-1} \pm 6.65 \!\cdot\! 10^{-3}$ & $3.1 \!\cdot\! 10^{-2} \pm 9.22 \!\cdot\! 10^{-3}$ \\
			n=6, p=3 & $1.15 \!\cdot\! 10^{-1} \pm 3.58 \!\cdot\! 10^{-2}$ & $1.41 \!\cdot\! 10^{0} \pm 4.96 \!\cdot\! 10^{-3}$ & $7.07 \!\cdot\! 10^{-1} \pm 8.16 \!\cdot\! 10^{-2}$ \\
			\bottomrule
		\end{tabular}
	\end{center}
\end{table*}

\subsection{Varying Number of Charts}

In the main paper we generally consider the case where there is a single chart or there is a fixed number of charts. Here we perform some additional experiments on the effect of varying number of charts. We always consider the setting where the total number of layers is constant, such that using more charts imply using fewer numbers of layers for the individual flows.

We first consider Sphere-U, i.e. the uniform distribution over the two dimensional sphere embedded in three dimensional space. A single chart cannot cover the distribution, while two or more charts are enough. The results for the case where there are $4$ charts are directly taken from the main paper. The final result tables are shown in Table~\ref{tbl:varying1-0} and Table~\ref{tbl:varying1-1}. In general, the model's performances appear rather stable across the different numbers of charts.

\begin{table*}
	\begin{center}
	\caption{Evaluation metrics concerning modeling using varying number of charts on Sphere-U in the form of mean $\pm$ std, lower is better. The best are highlighted in bold.}
    \label{tbl:varying1-0}
		\begin{tabular}{lll}
			\toprule
			Num charts & Recons & $\mathcal{W}$ \\
			\midrule
			2 & $6.95 \!\cdot\! 10^{-6} \pm 1.66 \!\cdot\! 10^{-6}$ & $1.33 \!\cdot\! 10^{-1} \pm 1.23 \!\cdot\! 10^{-2}$ \\
			3 & $3.03 \!\cdot\! 10^{-6} \pm 9.82 \!\cdot\! 10^{-7}$ & $1.31 \!\cdot\! 10^{-1} \pm 1.28 \!\cdot\! 10^{-2}$ \\
			4 & $2.22 \!\cdot\! 10^{-6} \pm 7.44 \!\cdot\! 10^{-7}$ & $\mathbf{1.29 \!\cdot\! 10^{-1}} \pm 1.29 \!\cdot\! 10^{-2}$ \\
			6 & $\mathbf{1.18 \!\cdot\! 10^{-6}} \pm 3.46 \!\cdot\! 10^{-7}$ & $1.33 \!\cdot\! 10^{-1} \pm 1.53 \!\cdot\! 10^{-2}$ \\
			\bottomrule
		\end{tabular}
	\end{center}
\end{table*}

\begin{table*}
	\begin{center}
	\caption{Evaluation metrics concerning geometry using varying number of charts on Sphere-U in the form of mean $\pm$ std, lower is better. The best are highlighted in bold.}
    \label{tbl:varying1-1}
		\begin{tabular}{llll}
			\toprule
			Num charts & Exps & Logs & Dists \\
			\midrule
			2 & $3.11 \!\cdot\! 10^{-2} \pm 9.63 \!\cdot\! 10^{-3}$ & $8.59 \!\cdot\! 10^{-2} \pm 6.7 \!\cdot\! 10^{-3}$ & $1.79 \!\cdot\! 10^{-4} \pm 2.75 \!\cdot\! 10^{-5}$ \\
			3 & $\mathbf{1.66 \!\cdot\! 10^{-3}} \pm 9.32 \!\cdot\! 10^{-4}$ & $\mathbf{6.56 \!\cdot\! 10^{-2}} \pm 1.23 \!\cdot\! 10^{-2}$ & $\mathbf{1.58 \!\cdot\! 10^{-4}} \pm 2.99 \!\cdot\! 10^{-5}$ \\
			4 & $2.28 \!\cdot\! 10^{-2} \pm 1.52 \!\cdot\! 10^{-2}$ & $1.32 \!\cdot\! 10^{-1} \pm 2.86 \!\cdot\! 10^{-2}$ & $3.77 \!\cdot\! 10^{-4} \pm 1.27 \!\cdot\! 10^{-4}$ \\
			6 & $8.91 \!\cdot\! 10^{-3} \pm 1.14 \!\cdot\! 10^{-2}$ & $1.45 \!\cdot\! 10^{-1} \pm 3.63 \!\cdot\! 10^{-2}$ & $7.76 \!\cdot\! 10^{-4} \pm 2.45 \!\cdot\! 10^{-4}$ \\
			\bottomrule
		\end{tabular}
	\end{center}
\end{table*}

We next consider the case where the target distribution is a part of the uniform distribution over the two dimensional sphere in the three dimensional Euclidean space. In this case, a single chart is already enough. Nevertheless, we consider the case where we use two charts. Since the training objective is based on the entire mixture model, the individual flows are encouraged to cover different areas without being degenerate. As shown in Table~\ref{tbl:varying2}, using \mult\ results in highly competitive performances.

\begin{table*}
	\begin{center}
	\caption{Evaluation metrics on a part of Sphere-U in the form of mean $\pm$ std, lower is better. The best are highlighted in bold.}
    \label{tbl:varying2}
		\begin{tabular}{lll}
			\toprule
			Method & Recons & $\mathcal{W}$ \\
			\midrule
			\singlem & $\mathbf{1.04 \!\cdot\! 10^{-7}} \pm 3.88 \!\cdot\! 10^{-8}$ & $5.55 \!\cdot\! 10^{-2} \pm 8.1 \!\cdot\! 10^{-3}$ \\
			\single & $1.52 \!\cdot\! 10^{-6} \pm 4.14 \!\cdot\! 10^{-7}$ & $6.21 \!\cdot\! 10^{-2} \pm 8.49 \!\cdot\! 10^{-3}$ \\
			\mult & $1.35 \!\cdot\! 10^{-6} \pm 2.85 \!\cdot\! 10^{-7}$ & $\mathbf{5.41 \!\cdot\! 10^{-2}} \pm 6.64 \!\cdot\! 10^{-3}$ \\
			\bottomrule
		\end{tabular}
	\end{center}
\end{table*}

In general, as long as the number of charts exceeds what is required to cover the space, it is reasonable to expect that each chart can more easily cover a local region, where the flows cover approximately equal amount of masses. A such, as long as each individual flow has a sufficient number of layers to model part of the target distribution, we can expect the performance to remain stable. In fact, the empirical observation that the performances are stable across the different numbers of layers suggest that our method consistently learns the same underlying manifold, demonstrating stability and robustness.

\subsection{Consistency of the Charts}

Our training objective involves the reconstruction losses of the charts at the points where they are reasonably responsible. As such, it can be expected that the different charts yield rather similar reconstructions at the points where they are responsible for. We empirically verify this in Table~\ref{tbl:boundary_diffs}, where we measure the mean squared differences between every pair of reconstructions generated by two flows that have responsibilities over $0.1$ for each data point from the test dataset, reported on the sphere and torus datasets that were considered in the main paper. Note that in the main paper we compared the overall reconstructions of the flows with the ground truth, while here we are comparing the reconstructions generated by the individual flows against each other. We observe that the inconsistencies among the individual flows are generally small.

\begin{table*}
	\begin{center}
	\caption{Differences between reconstructions of different charts on boundaries in the form of mean $\pm$ std. The differences among the charts are generally small.}
    \label{tbl:boundary_diffs}
		\begin{tabular}{ll}
			\toprule
			Data & Diffs \\
			\midrule
			Sphere-U & $1.15 \!\cdot\! 10^{-5} \pm 5.8 \!\cdot\! 10^{-6}$ \\
			Torus-U & $6.04 \!\cdot\! 10^{-5} \pm 8.82 \!\cdot\! 10^{-6}$ \\
			Sphere-M & $1.41 \!\cdot\! 10^{-5} \pm 1.3 \!\cdot\! 10^{-6}$ \\
			Torus-M & $1.57 \!\cdot\! 10^{-4} \pm 3.16 \!\cdot\! 10^{-5}$ \\
			\bottomrule
		\end{tabular}
	\end{center}
\end{table*}

\subsection{Direct MLE Compared with EM}
\label{sec:app-em-mle-exps}

In general, direct MLE and EM achieves highly comparable performances. We report experimental results on sphere and torus across $3$ independent runs in terms of Recons, $\mathcal{W}$, Exps, Logs and Dists in Table~\ref{tbl:mle_em_recons}, Table~\ref{tbl:mle_em_recons}, Table~\ref{tbl:mle_em_exp}, Table~\ref{tbl:mle_em_log} and Table~\ref{tbl:mle_em_dist}.

\begin{table*}
	\begin{center}
	\caption{Reconstructions of MLE compared with EM across different data.}
	\label{tbl:mle_em_recons}
        \begin{tabular}{lllll}
			\toprule
			Alg & Sphere-U & Torus-U & Sphere-M & Torus-M \\
			\midrule
			MLE & $\mathbf{1.68 \!\cdot\! 10^{-6}} \pm 4.29 \!\cdot\! 10^{-7}$ & $4.65 \!\cdot\! 10^{-5} \pm 4.18 \!\cdot\! 10^{-5}$ & $1.71 \!\cdot\! 10^{-6} \pm 1.58 \!\cdot\! 10^{-7}$ & $\mathbf{5.73 \!\cdot\! 10^{-5}} \pm 2.62 \!\cdot\! 10^{-5}$ \\
			EM & $2.22 \!\cdot\! 10^{-6} \pm 7.44 \!\cdot\! 10^{-7}$ & $\mathbf{2.25 \!\cdot\! 10^{-5}} \pm 1.99 \!\cdot\! 10^{-6}$ & $\mathbf{1.58 \!\cdot\! 10^{-6}} \pm 2.42 \!\cdot\! 10^{-7}$ & $6.9 \!\cdot\! 10^{-5} \pm 2.96 \!\cdot\! 10^{-5}$ \\
			\bottomrule
		\end{tabular}
	\end{center}
\end{table*}

\begin{table*}
	\begin{center}
	\caption{Wasserstein distances of MLE compared with EM across different data.}
    \label{tbl:mle_em_w}
		\begin{tabular}{lllll}
			\toprule
			Alg & Sphere-U & Torus-U & Sphere-M & Torus-M \\
			\midrule
			MLE & $\mathbf{1.29 \!\cdot\! 10^{-1}} \pm 1.24 \!\cdot\! 10^{-2}$ & $3.26 \!\cdot\! 10^{-1} \pm 6.17 \!\cdot\! 10^{-2}$ & $\mathbf{1.34 \!\cdot\! 10^{-1}} \pm 1.62 \!\cdot\! 10^{-2}$ & $\mathbf{3.09 \!\cdot\! 10^{-1}} \pm 5.93 \!\cdot\! 10^{-2}$ \\
			EM & $1.29 \!\cdot\! 10^{-1} \pm 1.29 \!\cdot\! 10^{-2}$ & $\mathbf{2.46 \!\cdot\! 10^{-1}} \pm 2.55 \!\cdot\! 10^{-2}$ & $1.34 \!\cdot\! 10^{-1} \pm 1.66 \!\cdot\! 10^{-2}$ & $3.15 \!\cdot\! 10^{-1} \pm 5.93 \!\cdot\! 10^{-2}$ \\
			\bottomrule
		\end{tabular}
	\end{center}
\end{table*}

\begin{table*}
	\begin{center}
	\caption{Quality of exponential map solutions of MLE compared with EM across different data.}
    \label{tbl:mle_em_exp}
		\begin{tabular}{lllll}
			\toprule
			Alg & Sphere-U & Torus-U & Sphere-M & Torus-M \\
			\midrule
			MLE & $\mathbf{9.05 \!\cdot\! 10^{-3}} \pm 7.53 \!\cdot\! 10^{-3}$ & $1.07 \!\cdot\! 10^{0} \pm 5.73 \!\cdot\! 10^{-1}$ & $\mathbf{7.21 \!\cdot\! 10^{-3}} \pm 3.68 \!\cdot\! 10^{-3}$ & $\mathbf{1.29 \!\cdot\! 10^{0}} \pm 2.14 \!\cdot\! 10^{-1}$ \\
			EM & $2.28 \!\cdot\! 10^{-2} \pm 1.52 \!\cdot\! 10^{-2}$ & $\mathbf{1.17 \!\cdot\! 10^{-1}} \pm 3.99 \!\cdot\! 10^{-2}$ & $9.19 \!\cdot\! 10^{-3} \pm 5.65 \!\cdot\! 10^{-3}$ & $1.5 \!\cdot\! 10^{0} \pm 3.82 \!\cdot\! 10^{-1}$ \\
			\bottomrule
		\end{tabular}
	\end{center}
\end{table*}

\begin{table*}
	\begin{center}
	\caption{Quality of logarithmic map solutions of MLE compared with EM across different data.}
    \label{tbl:mle_em_log}
		\begin{tabular}{lllll}
			\toprule
			Alg & Sphere-U & Torus-U & Sphere-M & Torus-M \\
			\midrule
			MLE & $1.37 \!\cdot\! 10^{-1} \pm 2.59 \!\cdot\! 10^{-2}$ & $\mathbf{3.69 \!\cdot\! 10^{0}} \pm 1.76 \!\cdot\! 10^{0}$ & $\mathbf{1.55 \!\cdot\! 10^{-1}} \pm 4.08 \!\cdot\! 10^{-2}$ & $\mathbf{2.23 \!\cdot\! 10^{0}} \pm 4.87 \!\cdot\! 10^{-1}$ \\
			EM & $\mathbf{1.32 \!\cdot\! 10^{-1}} \pm 2.86 \!\cdot\! 10^{-2}$ & $6.41 \!\cdot\! 10^{0} \pm 1.02 \!\cdot\! 10^{0}$ & $1.67 \!\cdot\! 10^{-1} \pm 5.72 \!\cdot\! 10^{-2}$ & $2.32 \!\cdot\! 10^{0} \pm 6.07 \!\cdot\! 10^{-1}$ \\
			\bottomrule
		\end{tabular}
	\end{center}
\end{table*}

\begin{table*}
	\begin{center}
	\caption{Quality of distance solutions of MLE compared with EM across different data.}
    \label{tbl:mle_em_dist}
		\begin{tabular}{lllll}
			\toprule
			Alg & Sphere-U & Torus-U & Sphere-M & Torus-M \\
			\midrule
			MLE & $4.3 \!\cdot\! 10^{-4} \pm 1.61 \!\cdot\! 10^{-4}$ & $\mathbf{1.03 \!\cdot\! 10^{-1}} \pm 8.1 \!\cdot\! 10^{-2}$ & $8.72 \!\cdot\! 10^{-4} \pm 3.16 \!\cdot\! 10^{-4}$ & $5.18 \!\cdot\! 10^{-2} \pm 1.6 \!\cdot\! 10^{-2}$ \\
			EM & $\mathbf{3.77 \!\cdot\! 10^{-4}} \pm 1.27 \!\cdot\! 10^{-4}$ & $4.34 \!\cdot\! 10^{-1} \pm 1.38 \!\cdot\! 10^{-1}$ & $\mathbf{8.2 \!\cdot\! 10^{-4}} \pm 4.1 \!\cdot\! 10^{-4}$ & $\mathbf{5.13 \!\cdot\! 10^{-2}} \pm 1.71 \!\cdot\! 10^{-2}$ \\
			\bottomrule
		\end{tabular}
	\end{center}
\end{table*}

\subsection{Initialization Strategy for Geodesic Solver}

When solving for geodesics by optimizing for the curve, it is clear that the initialization strategy can make a difference. While in the main paper we directly use Stochman's default initialization strategy, below we consider an alternative where we attempt to achieve data-informed initializations. Specifically, we build a K-Nearest Neighbor graph and use the solution based on the graph to initialize the spline. We show the qualities of logarithmic map solutions in Table~\ref{tbl:graph_log} and qualities of distance solutions in Table~\ref{tbl:graph_dist}. Interestingly, the graph-based solution often yields worse results than the naive approach.

\begin{table*}
	\begin{center}
	\caption{Quality of logarithmic map solutions of graph-based initialization compared with vanilla across different data.}
    \label{tbl:graph_log}
		\begin{tabular}{lll}
			\toprule
			Data & Raw & Graph \\
			\midrule
			Sphere-U & $\mathbf{1.32 \!\cdot\! 10^{-1}} \pm 2.86 \!\cdot\! 10^{-2}$ & $1.6 \!\cdot\! 10^{-1} \pm 1.16 \!\cdot\! 10^{-2}$ \\
			Torus-U & $6.41 \!\cdot\! 10^{0} \pm 1.02 \!\cdot\! 10^{0}$ & $\mathbf{2.84 \!\cdot\! 10^{0}} \pm 9.2 \!\cdot\! 10^{-2}$ \\
			Sphere-M & $\mathbf{1.67 \!\cdot\! 10^{-1}} \pm 5.72 \!\cdot\! 10^{-2}$ & $7.86 \!\cdot\! 10^{-1} \pm 7.79 \!\cdot\! 10^{-2}$ \\
			Torus-M & $\mathbf{2.32 \!\cdot\! 10^{0}} \pm 6.07 \!\cdot\! 10^{-1}$ & $4.76 \!\cdot\! 10^{0} \pm 8.04 \!\cdot\! 10^{-1}$ \\
			\bottomrule
		\end{tabular}
	\end{center}
\end{table*}

\begin{table*}
	\begin{center}
	\caption{Quality of distance solutions of graph-based initialization compared with vanilla across different data.}
    \label{tbl:graph_dist}
		\begin{tabular}{lll}
			\toprule
			Data & Raw & Graph \\
			\midrule
			Sphere-U & $\mathbf{3.77 \!\cdot\! 10^{-4}} \pm 1.27 \!\cdot\! 10^{-4}$ & $3.4 \!\cdot\! 10^{-3} \pm 1.82 \!\cdot\! 10^{-4}$ \\
			Torus-U & $4.34 \!\cdot\! 10^{-1} \pm 1.38 \!\cdot\! 10^{-1}$ & $\mathbf{4.37 \!\cdot\! 10^{-2}} \pm 8.06 \!\cdot\! 10^{-3}$ \\
			Sphere-M & $\mathbf{8.2 \!\cdot\! 10^{-4}} \pm 4.1 \!\cdot\! 10^{-4}$ & $3.89 \!\cdot\! 10^{-2} \pm 5.9 \!\cdot\! 10^{-3}$ \\
			Torus-M & $\mathbf{5.13 \!\cdot\! 10^{-2}} \pm 1.71 \!\cdot\! 10^{-2}$ & $1.27 \!\cdot\! 10^{-1} \pm 5.16 \!\cdot\! 10^{-2}$ \\
			\bottomrule
		\end{tabular}
	\end{center}
\end{table*}

\subsection{Computational Time}

We report the running times of the different algorithms on sphere and torus across $3$ runs in Table~\ref{tbl:comp_times}. \mult often results in comparable running times as \single.

\begin{table*}
	\begin{center}
	\caption{Computational times of the different algorithms.}
    \label{tbl:comp_times}
		\begin{tabular}{llll}
			\toprule
			Method & Single-M & Single & Mult \\
			\midrule
			Sphere-U & [2637.04, 1967.51, 1537.88] & [8887.35, 8376.48, 3455.06] & [9924.72, 10077.32, 7130.9] \\
			Torus-U & [2972.0, 1033.06, 1417.95] & [5926.32, 3846.62, 8409.61] & [12012.15, 9163.91, 8925.49] \\
			Sphere-M & [1622.29, 2792.39, 2560.15] & [12085.48, 9273.44, 4849.75] & [9533.42, 10608.23, 8715.19] \\
			Torus-M & [2653.07, 1713.29, 2954.03] & [17894.75, 7220.11, 19395.33] & [12857.94, 15841.65, 11724.81] \\
			\bottomrule
		\end{tabular}
	\end{center}
\end{table*}

\subsection{Noise and Outliers}

We consider the scenario where the data is corrupted by noise and outliers. Specifically, we consider the uniform distribution on two dimensional sphere. We add independent Gaussian or Student-t noise with degrees of freedom $3$ to the both the training set and the validation set, while using noiseless test set. We show the results of \singlem, \single\ and \mult\ in terms of modeling and geometry in Table~\ref{tbl:noise_singlem_0}, Table~\ref{tbl:noise_single_0}, Table~\ref{tbl:noise_mult_0}, Table~\ref{tbl:noise_singlem_1}, Table~\ref{tbl:noise_single_1} and Table~\ref{tbl:noise_mult_1}, respectively. The models are still capable of learning the underlying geometry and topology with similar trends of performances, though naturally the performances degrade as the amount of noise and outliers increases.

\begin{table*}
	\begin{center}
	\caption{Evaluation metrics concerning modeling of \singlem\ on data with noise and outliers in the form of mean $\pm$ std, lower is better.}
    \label{tbl:noise_singlem_0}
		\begin{tabular}{lll}
			\toprule
			Data & Recons & $\W$ \\
			\midrule
			Raw & $8.61 \!\cdot\! 10^{-5} \pm 4.9 \!\cdot\! 10^{-5}$ & $1.29 \!\cdot\! 10^{-1} \pm 1.24 \!\cdot\! 10^{-2}$ \\
			N 0.01 & $2.26 \!\cdot\! 10^{-4} \pm 2.43 \!\cdot\! 10^{-4}$ & $1.32 \!\cdot\! 10^{-1} \pm 1.18 \!\cdot\! 10^{-2}$ \\
			N 0.03 & $1.67 \!\cdot\! 10^{-4} \pm 1.25 \!\cdot\! 10^{-4}$ & $1.3 \!\cdot\! 10^{-1} \pm 1.04 \!\cdot\! 10^{-2}$ \\
			N 0.1 & $3.45 \!\cdot\! 10^{-4} \pm 6.62 \!\cdot\! 10^{-5}$ & $1.37 \!\cdot\! 10^{-1} \pm 1.17 \!\cdot\! 10^{-2}$ \\
			T 0.01 & $1.18 \!\cdot\! 10^{-4} \pm 1.01 \!\cdot\! 10^{-4}$ & $1.34 \!\cdot\! 10^{-1} \pm 9.25 \!\cdot\! 10^{-3}$ \\
			T 0.03 & $3.27 \!\cdot\! 10^{-4} \pm 2.86 \!\cdot\! 10^{-4}$ & $1.33 \!\cdot\! 10^{-1} \pm 1.42 \!\cdot\! 10^{-2}$ \\
			T 0.1 & $4.6 \!\cdot\! 10^{-3} \pm 5.96 \!\cdot\! 10^{-3}$ & $1.45 \!\cdot\! 10^{-1} \pm 1.79 \!\cdot\! 10^{-2}$ \\
			\bottomrule
		\end{tabular}
	\end{center}
\end{table*}

\begin{table*}
	\begin{center}
	\caption{Evaluation metrics concerning modeling of \single\ on data with noise and outliers in the form of mean $\pm$ std, lower is better.}
    \label{tbl:noise_single_0}
		\begin{tabular}{lll}
			\toprule
			Data & Recons & $\W$ \\
			\midrule
			Raw & $6.81 \!\cdot\! 10^{-5} \pm 4.66 \!\cdot\! 10^{-5}$ & $1.35 \!\cdot\! 10^{-1} \pm 9.43 \!\cdot\! 10^{-3}$ \\
			N 0.01 & $1.8 \!\cdot\! 10^{-4} \pm 1.26 \!\cdot\! 10^{-4}$ & $6.51 \!\cdot\! 10^{1} \pm 2.43 \!\cdot\! 10^{2}$ \\
			N 0.03 & $1.4 \!\cdot\! 10^{-4} \pm 2.97 \!\cdot\! 10^{-5}$ & $1.38 \!\cdot\! 10^{-1} \pm 1.6 \!\cdot\! 10^{-2}$ \\
			N 0.1 & $6.83 \!\cdot\! 10^{-4} \pm 4.0 \!\cdot\! 10^{-4}$ & $1.47 \!\cdot\! 10^{-1} \pm 2.32 \!\cdot\! 10^{-2}$ \\
			T 0.01 & $1.8 \!\cdot\! 10^{-4} \pm 1.48 \!\cdot\! 10^{-4}$ & $1.34 \!\cdot\! 10^{-1} \pm 1.41 \!\cdot\! 10^{-2}$ \\
			T 0.03 & $1.85 \!\cdot\! 10^{-4} \pm 1.08 \!\cdot\! 10^{-4}$ & $1.36 \!\cdot\! 10^{-1} \pm 1.27 \!\cdot\! 10^{-2}$ \\
			T 0.1 & $1.11 \!\cdot\! 10^{-3} \pm 6.25 \!\cdot\! 10^{-4}$ & $1.44 \!\cdot\! 10^{-1} \pm 1.33 \!\cdot\! 10^{-2}$ \\
			\bottomrule
		\end{tabular}
	\end{center}
\end{table*}

\begin{table*}
	\begin{center}
	\caption{Evaluation metrics concerning modeling of \mult\ on data with noise and outliers in the form of mean $\pm$ std, lower is better.}
    \label{tbl:noise_mult_0}
		\begin{tabular}{lll}
			\toprule
			Data & Recons & $\W$ \\
			\midrule
			Raw & $2.22 \!\cdot\! 10^{-6} \pm 7.44 \!\cdot\! 10^{-7}$ & $1.29 \!\cdot\! 10^{-1} \pm 1.29 \!\cdot\! 10^{-2}$ \\
			N 0.01 & $8.4 \!\cdot\! 10^{-6} \pm 9.34 \!\cdot\! 10^{-7}$ & $1.36 \!\cdot\! 10^{-1} \pm 1.1 \!\cdot\! 10^{-2}$ \\
			N 0.03 & $1.42 \!\cdot\! 10^{-4} \pm 9.29 \!\cdot\! 10^{-6}$ & $1.32 \!\cdot\! 10^{-1} \pm 8.79 \!\cdot\! 10^{-3}$ \\
			N 0.1 & $8.55 \!\cdot\! 10^{-4} \pm 4.68 \!\cdot\! 10^{-5}$ & $1.62 \!\cdot\! 10^{-1} \pm 8.89 \!\cdot\! 10^{-3}$ \\
			T 0.01 & $8.1 \!\cdot\! 10^{-6} \pm 1.54 \!\cdot\! 10^{-6}$ & $1.32 \!\cdot\! 10^{-1} \pm 1.01 \!\cdot\! 10^{-2}$ \\
			T 0.03 & $1.14 \!\cdot\! 10^{-4} \pm 1.95 \!\cdot\! 10^{-5}$ & $1.39 \!\cdot\! 10^{-1} \pm 1.06 \!\cdot\! 10^{-2}$ \\
			T 0.1 & $6.97 \!\cdot\! 10^{-4} \pm 1.12 \!\cdot\! 10^{-4}$ & $1.81 \!\cdot\! 10^{-1} \pm 1.2 \!\cdot\! 10^{-2}$ \\
			\bottomrule
		\end{tabular}
	\end{center}
\end{table*}

\begin{table*}
	\begin{center}
	\caption{Evaluation metrics concerning geometry of \singlem\ on data with noise and outliers in the form of mean $\pm$ std, lower is better.}
    \label{tbl:noise_singlem_1}
		\begin{tabular}{llll}
			\toprule
			Data & Exps & Logs & Dists \\
			\midrule
			Raw & $1.86 \!\cdot\! 10^{3} \pm 2.62 \!\cdot\! 10^{3}$ & $8.58 \!\cdot\! 10^{-1} \pm 3.45 \!\cdot\! 10^{-1}$ & $2.35 \!\cdot\! 10^{-2} \pm 7.66 \!\cdot\! 10^{-3}$ \\
			N 0.01 & $4.94 \!\cdot\! 10^{16} \pm 6.99 \!\cdot\! 10^{16}$ & $1.28 \!\cdot\! 10^{0} \pm 2.84 \!\cdot\! 10^{-1}$ & $8.34 \!\cdot\! 10^{-2} \pm 6.0 \!\cdot\! 10^{-2}$ \\
			N 0.03 & $1.76 \!\cdot\! 10^{0} \pm 1.18 \!\cdot\! 10^{0}$ & $1.41 \!\cdot\! 10^{0} \pm 1.14 \!\cdot\! 10^{0}$ & $7.53 \!\cdot\! 10^{-2} \pm 9.07 \!\cdot\! 10^{-2}$ \\
			N 0.1 & $1.19 \!\cdot\! 10^{1} \pm 1.17 \!\cdot\! 10^{1}$ & $7.02 \!\cdot\! 10^{-1} \pm 9.09 \!\cdot\! 10^{-2}$ & $2.5 \!\cdot\! 10^{-2} \pm 9.8 \!\cdot\! 10^{-3}$ \\
			T 0.01 & $2.91 \!\cdot\! 10^{2} \pm 4.11 \!\cdot\! 10^{2}$ & $1.31 \!\cdot\! 10^{0} \pm 6.5 \!\cdot\! 10^{-1}$ & $1.07 \!\cdot\! 10^{-1} \pm 7.74 \!\cdot\! 10^{-2}$ \\
			T 0.03 & $2.0 \!\cdot\! 10^{18} \pm 2.83 \!\cdot\! 10^{18}$ & $9.31 \!\cdot\! 10^{-1} \pm 2.21 \!\cdot\! 10^{-1}$ & $4.4 \!\cdot\! 10^{-2} \pm 2.11 \!\cdot\! 10^{-2}$ \\
			T 0.1 & $4.9 \!\cdot\! 10^{-1} \pm 2.23 \!\cdot\! 10^{-1}$ & $1.16 \!\cdot\! 10^{0} \pm 4.37 \!\cdot\! 10^{-1}$ & $6.03 \!\cdot\! 10^{-2} \pm 3.76 \!\cdot\! 10^{-2}$ \\
			\bottomrule
		\end{tabular}
	\end{center}
\end{table*}

\begin{table*}
	\begin{center}
	\caption{Evaluation metrics concerning geometry of \single\ on data with noise and outliers in the form of mean $\pm$ std, lower is better.}
    \label{tbl:noise_single_1}
		\begin{tabular}{llll}
			\toprule
			Data & Exps & Logs & Dists \\
			\midrule
			Raw & $\mathit{1.91 \!\cdot\! 10^{24}} \pm 1.91 \!\cdot\! 10^{24}$ & $1.09 \!\cdot\! 10^{0} \pm 2.0 \!\cdot\! 10^{-1}$ & $5.22 \!\cdot\! 10^{-2} \pm 6.04 \!\cdot\! 10^{-3}$ \\
			N 0.01 & $1.43 \!\cdot\! 10^{13} \pm 2.02 \!\cdot\! 10^{13}$ & $7.86 \!\cdot\! 10^{-1} \pm 4.49 \!\cdot\! 10^{-2}$ & $3.75 \!\cdot\! 10^{-2} \pm 7.46 \!\cdot\! 10^{-3}$ \\
			N 0.03 & $7.96 \!\cdot\! 10^{0} \pm 1.06 \!\cdot\! 10^{1}$ & $6.63 \!\cdot\! 10^{-1} \pm 4.31 \!\cdot\! 10^{-2}$ & $1.89 \!\cdot\! 10^{-2} \pm 8.43 \!\cdot\! 10^{-3}$ \\
			N 0.1 & $1.69 \!\cdot\! 10^{1} \pm 2.37 \!\cdot\! 10^{1}$ & $7.72 \!\cdot\! 10^{-1} \pm 2.98 \!\cdot\! 10^{-1}$ & $2.82 \!\cdot\! 10^{-2} \pm 1.77 \!\cdot\! 10^{-2}$ \\
			T 0.01 & $1.51 \!\cdot\! 10^{9} \pm 2.09 \!\cdot\! 10^{9}$ & $6.9 \!\cdot\! 10^{-1} \pm 2.48 \!\cdot\! 10^{-1}$ & $2.58 \!\cdot\! 10^{-2} \pm 1.45 \!\cdot\! 10^{-2}$ \\
			T 0.03 & $2.95 \!\cdot\! 10^{2} \pm 4.18 \!\cdot\! 10^{2}$ & $8.95 \!\cdot\! 10^{-1} \pm 3.37 \!\cdot\! 10^{-1}$ & $5.87 \!\cdot\! 10^{-2} \pm 6.04 \!\cdot\! 10^{-2}$ \\
			T 0.1 & $7.95 \!\cdot\! 10^{1} \pm 1.12 \!\cdot\! 10^{2}$ & $1.03 \!\cdot\! 10^{0} \pm 7.82 \!\cdot\! 10^{-1}$ & $6.96 \!\cdot\! 10^{-2} \pm 8.21 \!\cdot\! 10^{-2}$ \\
			\bottomrule
		\end{tabular}
	\end{center}
\end{table*}

\begin{table*}
	\begin{center}
	\caption{Evaluation metrics concerning geometry of \mult\ on data with noise and outliers in the form of mean $\pm$ std, lower is better.}
    \label{tbl:noise_mult_1}
		\begin{tabular}{llll}
			\toprule
			Data & Exps & Logs & Dists \\
			\midrule
			Raw & $2.28 \!\cdot\! 10^{-2} \pm 1.52 \!\cdot\! 10^{-2}$ & $1.32 \!\cdot\! 10^{-1} \pm 2.86 \!\cdot\! 10^{-2}$ & $3.77 \!\cdot\! 10^{-4} \pm 1.27 \!\cdot\! 10^{-4}$ \\
			N 0.01 & $3.48 \!\cdot\! 10^{-2} \pm 3.55 \!\cdot\! 10^{-2}$ & $1.18 \!\cdot\! 10^{-1} \pm 3.17 \!\cdot\! 10^{-2}$ & $4.71 \!\cdot\! 10^{-4} \pm 3.5 \!\cdot\! 10^{-4}$ \\
			N 0.03 & $1.04 \!\cdot\! 10^{-1} \pm 1.42 \!\cdot\! 10^{-2}$ & $2.91 \!\cdot\! 10^{-1} \pm 2.59 \!\cdot\! 10^{-2}$ & $2.56 \!\cdot\! 10^{-3} \pm 4.33 \!\cdot\! 10^{-4}$ \\
			N 0.1 & $1.31 \!\cdot\! 10^{-1} \pm 1.42 \!\cdot\! 10^{-2}$ & $1.09 \!\cdot\! 10^{0} \pm 1.51 \!\cdot\! 10^{-2}$ & $1.58 \!\cdot\! 10^{-2} \pm 1.42 \!\cdot\! 10^{-3}$ \\
			T 0.01 & $3.9 \!\cdot\! 10^{-2} \pm 4.45 \!\cdot\! 10^{-2}$ & $1.09 \!\cdot\! 10^{-1} \pm 3.15 \!\cdot\! 10^{-2}$ & $4.07 \!\cdot\! 10^{-4} \pm 2.1 \!\cdot\! 10^{-4}$ \\
			T 0.03 & $1.91 \!\cdot\! 10^{-1} \pm 3.13 \!\cdot\! 10^{-2}$ & $3.09 \!\cdot\! 10^{-1} \pm 1.0 \!\cdot\! 10^{-2}$ & $1.88 \!\cdot\! 10^{-3} \pm 5.07 \!\cdot\! 10^{-4}$ \\
			T 0.1 & $2.43 \!\cdot\! 10^{-1} \pm 2.98 \!\cdot\! 10^{-2}$ & $1.26 \!\cdot\! 10^{0} \pm 3.25 \!\cdot\! 10^{-2}$ & $1.97 \!\cdot\! 10^{-2} \pm 2.7 \!\cdot\! 10^{-3}$ \\
			\bottomrule
		\end{tabular}
	\end{center}
\end{table*}

\section{PREVIOUS WORKS}

Previous works that study the estimation of geometric structures \citep{kruiff2024pullback,diepeveen2025score,sun2025geometry-autoencoder} typically employ a single chart. As such, pathological issues are bound to appear when the method is trained on a dataset that cannot be covered by a single chart. As an illustrative example, we consider \citet{diepeveen2025score}. We take their code and train a Riemannian autoencoder, using as data samples drawn from a circle in the two dimensional Euclidean space. A correct Riemannian autoencoder should identify the correct latent dimensionality, i.e. one. However, due to the topology mismatch, their method could not achieve this in practice: the learned variances for the two dimensions are around $0.589$ and $0.325$, respectively, which are still large.

\citet{rozo2025riemann2} employed a Bayesian approach, and considered the problem of learning the submanifolds of existing manifolds.

\end{document}